\documentclass[twoside]{article}

\usepackage{hyperref}
\usepackage[acronym]{glossaries}
\usepackage{amssymb,amsmath,amsthm}
\usepackage{nicefrac}
\usepackage[noend]{algpseudocode}
\usepackage[ruled,vlined,linesnumbered]{algorithm2e}
\usepackage{caption}
\usepackage{subcaption}
\usepackage{graphicx} 
\usepackage{cleveref}
\usepackage{booktabs}
\usepackage{enumitem}
\usepackage{cleveref}
\graphicspath{{figures/}} 

\newtheorem{theorem}{Theorem}

\newtheorem{lemma}{Lemma}
\newtheorem{remark}{Remark}
\newtheorem{fact}{Fact}
\Crefname{algorithm}{Alg.}{Algs.}
\Crefname{section}{Sec.}{Secs.}
\Crefname{appendix}{App.}{Apps.}

\usepackage[accepted]{aistats2021}

\usepackage{amsmath,amsfonts,bm,amsthm}

\newcommand{\E}{{\mathbb E}}

\newcommand{\rank}[1]{\operatorname{rank}\left(#1\right)}
\DeclareMathOperator*{\argmin}{arg\,min}
\newcommand{\iter}[2]{ {#1}^{(#2)}}

\newcommand{\R}{{\mathbb R}}

\def\mA{{\mathbf{A}}}
\def\mB{{\mathbf{B}}}
\def\mC{{\mathbf{C}}}

\def\mE{{\mathbf{E}}}

\def\mH{{\mathbf{H}}}
\def\mI{{\mathbf{I}}}

\def\mM{{\mathbf{M}}}

\def\mQ{{\mathbf{Q}}}
\def\mR{{\mathbf{R}}}
\def\mS{{\mathbf{S}}}

\def\mU{{\mathbf{U}}}
\def\mV{{\mathbf{V}}}

\def\mX{{\mathbf{X}}}
\def\mY{{\mathbf{Y}}}

\def\mSigma{{\bm{\Sigma}}}

\def\vzero{{\mathbf{0}}}

\def\veps{{\bm{\eps}}}

\def\vb{{\mathbf{b}}}
\def\vc{{\mathbf{c}}}

\def\vu{{\mathbf{u}}}

\def\vx{{\mathbf{x}}}
\def\vy{{\mathbf{y}}}
\def\vz{{\mathbf{z}}}

\def\eps{{\varepsilon}}

\def\svd{SVD}
\def\lowner{L\"owner}
\def\lp{\ell_p} 
\newcommand{\levScores}[1][]{%
    \ifthenelse{\isempty{#1}}{$\lp{}$-leverage scores}{$\ell_{#1}$-leverage scores}
}
\newcommand{\norm}[2]{\left\| #1 \right\|_{#2} }
\newcommand{\covMat}[1]{ {#1}^{\top} {#1}}
\newcommand{\inv}[1]{ \left( #1 \right)^{-1}}
\newcommand{\eye}[1]{ \mI_{#1}}
\newcommand{\opt}[1]{ {#1}^{\star}}
\newcommand{\sol}[2]{ {#1}_{#2}}

\newcommand{\Normal}[2]{ \mathcal{N} \left( {#1}, {#2} \right) }

\def\bias{\textsf{bias}}
\def\var{\textsf{var}}
\def\mse{\textsf{MSE}}
\def\tr{\textsf{trace}}

\newcommand{\RandCovMat}[2]{ {#1}^{\top}{#2}^{\top}{#2} {#1}}
\def\SkTime{T_{\textsf{sketch}}}
\newcommand{\fd}[2]{  {#1}=\textsf{FD}\left( {#2} \right) } 
\newcommand{\rfd}[2]{  {#1}=\textsf{RFD}\left( {#2} \right) } 
\newcommand{\xfd}[2]{  {#1}=\textsf{Sk}\left( {#2} \right) } 

\newacronym{ols}{OLS}{Ordinary Least Squares}
\newacronym{rr}{RR}{Ridge Regression}
\newacronym{fd}{FD}{Frequent Directions}
\newacronym{fdrr}{FDRR}{Frequent Directions Ridge Regression}
\newacronym{rfdrr}{RFDRR}{Robust Frequent Directions Ridge Regression}
\newacronym{ifdrr}{iFDRR}{Iterative Frequent Directions Ridge Regression}
\newacronym{mse}{MSE}{Mean Squared Error}
\newacronym{rfd}{RFD}{Robust Frequent Directions}
\newacronym{ml}{ML}{Machine Learning}
\newacronym{sjlt}{SJLT}{Sparse Johnson Lindenstrauss Transform}
\newacronym{ihs}{IHS}{Iterative Hessian Sketch}

\usepackage[round]{natbib}

\begin{document}

%

%

\twocolumn[

\aistatstitle{Ridge Regression with Frequent Directions: Statistical and 
Optimization Perspectives}

\aistatsauthor{Charlie Dickens}

\aistatsaddress{University of Warwick} ]

\begin{abstract}
Despite its impressive theory \& practical performance, Frequent Directions (\acrshort{fd})
has not been widely adopted for large-scale regression tasks.
Prior work has shown randomized sketches 
(i) perform worse in estimating the covariance matrix of the data than \acrshort{fd};
(ii) incur high error when estimating the bias and/or variance on sketched ridge regression.
We give the first constant factor relative error bounds on the bias \& variance for sketched ridge 
regression using \acrshort{fd}.
We complement these statistical results by showing that \acrshort{fd} can be used in 
the optimization setting through an iterative
scheme which yields high-accuracy solutions.
This improves on randomized approaches which need to compromise the need for a new sketch every 
iteration with speed of convergence.
In both settings, we also show using \emph{Robust Frequent Directions} 
further enhances performance.
\end{abstract}

\section{Introduction}
\label{sec:introduction}

Ridge regression (\acrshort{rr}) has become a key tool in data analysis but it is resource intensive to solve at 
large scale and in high dimensions.
Recall that the \acrshort{rr} problem is to return:
\begin{equation}
    \min_{\vx \in \R^{d}} \left\{ f(\vx) = \frac{1}{2}\|\mA \vx - \vb \|_2^2
    + \frac{\gamma}{2}\|\vx\|_2^2 \right\}
    \label{eq:ridge-regression}
\end{equation}
Solving \eqref{eq:ridge-regression} when $n > d$ by the \svd{}  (or other related 
decompositions) requires $O(nd^2)$ time and $O(d^2)$ space.
These complexities are not acceptable given the scale of modern data.

A crucial quantity in both solving and approximating \acrshort{rr} is the 
\emph{Hessian}\footnote{Due to the fact it is the matrix of second derivatives of 
$f(\vx)$ in \eqref{eq:ridge-regression}.
It is composed of the data covariance $\covMat{\mA}$ and a regularization 
term $\gamma \eye{d}$.}
matrix $\mH_{\gamma} = \covMat{\mA} + \gamma \eye{d}$.
Maintaining $\mH_{\gamma}$ exactly by rank-one updates of the samples 
costs $O(nd^{\omega -1})$ time and $O(d^2)$ space so offers little overall benefit.
Speeding up this computation has been studied through \emph{randomized} matrix sketching techniques
which estimate $\mH_{\gamma}$ through 
$\tilde{\mH}_{\gamma} = \RandCovMat{\mA}{\mS} + \gamma \eye{d}$.
Provided that $\mS \in \R^{m \times n}$ is sampled from a suitable distribution
(details in \citep{woodruff2014sketching, drineas2016randnla}), then 
$\tilde{\mH}_{\gamma}$ is a good proxy for $\mH_{\gamma}$.
The computational savings come when $\mS$ can be applied to input $\mA$ quickly or implicitly as the
data is read.

The exact solution to \eqref{eq:ridge-regression} is given in 
\eqref{eq:rr-exact}.
There are two central ``one-shot'' methods to approximate \eqref{eq:ridge-regression}: \emph{Classical \eqref{eq:classical-sketch-def}}
\citep{avron2016sharper}
and \emph{Hessian \eqref{eq:hessian-sketch-def}} 
\citep{pilanci2015randomized} Sketching:
\vspace*{-2mm}
\begin{align}
    \opt{\vx} &= \inv{\covMat{\mA} + \gamma \eye{d}} \mA^{\top} \vy
    \label{eq:rr-exact}\\
    \vx^C     &= 
    \inv{\mA^{\top} \mS^{\top} \mS \mA + \gamma \eye{d}} 
    \mA^{\top} \mS^{\top} \mS \vy \label{eq:classical-sketch-def}\\
    \vx^H     &= 
     \inv{\mA^{\top} \mS^{\top} \mS \mA + \gamma \eye{d}} \mA^{\top} \vy\label{eq:hessian-sketch-def}
     \vspace*{-2mm}
\end{align} 
If sketching the data to obtain $\mS \mA$ takes time $\SkTime,$ then approximating
\eqref{eq:ridge-regression} is $O(\SkTime + md^2)$ time.
As\footnote{The $\tilde{O}$ notation suppresses lower order and 
failure probability terms} $m = \tilde{O}(d \operatorname{poly} \log(d))$ the space grows as $\tilde{O}(d^2)$.

When ridge regression is practical, one often finds
redundancy in the spectrum of high-dimensional input data.
Hence, it would be ideal to perform an online or streaming variant of \svd{} keeping only the
informative parts of the spectrum.
Unfortunately, greedy heuristics \citep{brand2002incremental} can be shown to perform arbitrarily badly \citep{huang2018near}.
\cite{liberty2013simple} introduced \acrfull{fd}
for exactly this problem;
to find a matrix summary $\mB \in \R^{m \times d}$ that well approximates the information one would 
obtain from performing an \svd{} of $\mA$.
Therefore, \acrshort{fd} is a natural candidate sketch for approximating ridge regression.

Frequent Directions is an orthogonal approach to randomized matrix sketching.
\acrshort{fd} \emph{deterministically} updates the top singular directions observed in the data stream, keeping only the most 
important (or the most frequently occurring).
In \cite{ghashami2016frequent}, compelling evidence was given that showed
\acrshort{fd} more accurately approximates $\covMat{\mA}$ at a given projection
dimension $m$ than randomized methods.
\acrshort{fd} is also a mergeable summary \citep{agarwal2013mergeable}, and can be adapted
to sparse data \citep{ghashami2016efficient}.
Given that $\mH_{\gamma} = \covMat{\mA} + \gamma \eye{d}$ is a fundamental operator in \acrshort{rr},
one would hope that \acrshort{fd} can be used as the sketch here, rather than random projection.
Indeed, this is shown in \cite{shi2020deterministic} who use
$\hat{\mH} = \covMat{\mB} + \gamma \eye{d}$ to approximate $\mH_{\gamma}$.
If the interplay between the regularisation $\gamma$ and the approximation error from \acrshort{fd} are
correctly balanced, then $\opt{\vx}$ can be reasonably approximated.
We refer to this approach as \emph{\acrfull{fdrr}} (\Cref{alg:fdrr}, \Cref{app:fd-properties})
returning $\hat{\vx} = \hat{\mH}^{-1} \mA^{\top}\vy$.

However, it remains the case that despite being a high-quality sketch, \acrshort{fd} is under-exploited
in regression tasks.
Our motivation is to better understand how \acrshort{fd} can be used in regression and what properties
it preserves.
We are interested in the following questions which prior work has failed to address:
\begin{enumerate}
    \vspace*{-4mm}
    \item{\textbf{Statistical model estimation.} Ridge regression is often studied under a linear model with a ground truth vector $\vx_0$
    that describes the behavior of the data.
    If \acrshort{fd} is employed as the sketch, then how does this distort the bias and variance 
    of the returned weights compared to the ``optimal'' bias and variance in recovering 
    $\vx_0$?}
    \vspace*{-2.5mm}
    \item{
    \textbf{Solution estimation.}  Can the coarse approximation of $\opt{\vx}$ from \cite{shi2020deterministic} be bootstrapped
    to obtain a high accuracy estimate $\hat{\vx}$?
    }
\end{enumerate}
\vspace*{-3mm}
These two questions underpin complementary perspectives commonly found in the machine 
learning community.
The former is a \emph{statistical perspective} while the latter is an \emph{optimization perspective}.
It is argued in \cite{wang2017sketched} that \emph{both} are of importance in theory and practice
depending on the application.
The statistical perspective is relevant in machine learning when the approximate solution $\hat{\vx}$
is used as a proxy for the optimal weights $\opt{\vx}$ which are too expensive to obtain.
Meanwhile, the optimization perspective is useful when one wishes to understand how sequentially 
refining expensive iterations can lead to better estimates of the solution vector.
\vspace*{-3mm}
\subsection{Contributions}
Our contributions are two-fold:
\begin{enumerate}
    \vspace*{-3mm}
    \item{\textbf{Statistical results:} we give the first analysis for 
    \acrshort{fdrr} under a linear model.
    We provide constant factor relative error bounds on the bias, variance,
    and mean-squared error (\acrshort{mse}).
    For a $\theta \in (0,1)$, 
    we find that 
    $ ({1-\theta} )\norm{\bias(\opt{\vx})}{2}^{2} \le 
    \norm{\bias(\hat{\vx})}{2}^{2} \le
    \nicefrac{\norm{\bias(\opt{\vx})}{2}^{2}}{1-\theta}$, 
    likewise for trace of variance and \acrshort{mse}.
    We show that using the more accurate \emph{Robust Frequent Directions} \citep{huang2018near}
    improves this to a $1-\theta'$ approximation for $\theta' < \theta$.}
    \vspace*{-2.5mm}
    \item{\textbf{Optimization results:} 
    we present the first analysis of \acrshort{fd} in an iterative scheme 
    to obtain \emph{high-quality solution estimation}.
    We show that
    $t$ iterates
    $\vx^{(t)} = \vx^{(t-1)} - \hat{\mH}^{-1} \nabla f(\vx^{(t)})$ yields weights $\hat{\vx} = \vx^{(t)}$
    satisfying $\|\hat{\vx} - \opt{\vx}\|_2 \le \zeta^t  \|\opt{\vx}\|_2 $.
    This can substantially improve the one-shot sketch estimate
    $\hat{\vx} = \vx^{(1)}$
    of \cite{shi2020deterministic} even if $t$ is small or moderate.
    }
\end{enumerate}
\vspace*{-4mm}
Although these results are simple, there are significant practical implications.
From the statistical side, our results show that \acrshort{fdrr}
returns weights which are much more faithful to the underlying model 
than randomized sketching: this is highlighted in Table \ref{tab:methods}.
If the bias-variance tradeoff is a key concern then \acrshort{fdrr} 
should be preferred to using random projections.
Secondly, on the optimization side, we show the existence of small-space deterministic 
preconditioners which can be iteratively used to refine the estimates to ridge regression.
The significance of our results is that \acrshort{fdrr} requires space only $O(md)$ for $m < d$
by storing $\mB \in \R^{m \times d}$ and some extra information such as $\mA^{\top}\vy$.
Consequently, \acrshort{fdrr} can operate in higher dimensions than
randomized methods which need $\mS \mA \in \R^{\tilde{O}(d) \times d}$.

\subsection{Related Work}
\begin{table*}[!ht]
\centering
\begin{tabular}{@{}lllllll@{}}
\toprule
Method               & \multicolumn{2}{l}{$\|\bias(\hat{\vx})\|_2^2$}                 & \multicolumn{2}{l}{$\tr(\var(\vx))$}                     & Time   & Space (num. rows $m$) \\ 
                     & LB           & UB                        & LB           & UB                           &        &      \\ \midrule
\acrshort{fdrr}      & $1 - \theta$ & 
$\nicefrac{1}{1 - \theta}$&$1 - \theta$  & $\nicefrac{1}{1 - \theta}$   &$O(ndm)$& $O(md)$\\
\acrshort{rfdrr}     & $1 - \theta'$& $\nicefrac{1}{1 -\theta'}$& $1 - \theta'$& $\nicefrac{1}{1 -\theta'}$   &$O(ndm)$& $O(md)$\\
Classical            &$\nicefrac{1}{(1+\rho)^2}$&$\nicefrac{1}{(1-\rho)^2}$&    $c_1\frac{n(1-\rho)}{m(1+\rho)^2}$        &
$c_2\frac{n(1+\rho)}{m(1-\rho)^2}$ & $O(\textsf{nnz}(\mA)) \sim O(nd^2)$     & $\tilde{O}(d \rho^{-2} \operatorname{poly} \log d)$  \\
Hessian              & 
$c'_1 \frac{\rho}{1+\rho}$             &$c'_2\rho(1+\rho)$
& $\nicefrac{1}{(1+\rho)^2}$ &$\nicefrac{1}{(1-\rho)^2}$            & 
$O(\textsf{nnz}(\mA)) \sim O(nd^2)$     & $\tilde{O}(d \rho^{-2} \operatorname{poly} \log d)$  \\ \bottomrule
\end{tabular}
\caption{
Lower (LB) \& upper (UB) bounds for bias \& variance of competing sketching methods.
Deterministic methods require
$m = \nicefrac{\|\mA - \mA_k\|_F^2}{(1-\sqrt{1-\theta})\gamma} + k < d$ rows.
Bounds for Classical/Hessian sketch are slight modifications of 
\citep{wang2017sketched} for $(1 \pm \rho)$-$\ell_2$ subspace embeddings
$\mS$.
Extra parameters are constants $c_1, c_2 \approx 1$ and singular values
$\sigma_j^2$ of the input data.
The constants $c'_1, c'_2$ are slightly more involved, \citep{wang2017sketched} should be consulted for the details.
}
\label{tab:methods}
\end{table*}
Although we are not the first to study \acrshort{fd} in regression tasks,
prior work has different motivations, presents \emph{complementary} results to ours, and thus uses
different techniques.
\cite{shi2020deterministic} introduced \acrshort{fdrr}, returning 
$\hat{\vx}$
which satisfies a coarse bound $\norm{\hat{\vx} - \opt{\vx}}{} \le \zeta \norm{\opt{\vx}}{}$ in ${O}(d/\zeta)$ space.
Let $m = O(1/\zeta) < d$ be the number of rows in (and the rank of) $\mB$.
Since $\hat{\mH} = \covMat{\mB} + \gamma \eye{d}$, \cite{shi2020deterministic} show that
using an eigendecomposition of $\hat{\mH}$ can be used to obtain $\hat{\mH}^{-1}$ in 
$O(md)$ which results in the first $o(d^2)$ streaming algorithm to estimate $\opt{\vx}$.

However, this results fails
to provide any information on the model estimation provided by $\hat{\vx}$ performs under the
linear model we study.
Until this work nothing was known about the statistical performance of 
sketched ridge regression using \acrshort{fd}.
We seek strong statistical guarantees on the bias and variance of $\hat{\vx}$
when compared to the same quantities had no sketching been performed.
Alternatively, \cite{huang2018near} propose using \acrshort{fd} for \emph{adversarial 
online learning} through an approximate Newton method.
Hence, their application and bounds are much different from ours; no bounds on
the solution estimation $\norm{\hat{\vx} - \opt{\vx}}{2}$ are provided.

Randomized approaches for sketched ridge regression from the statistical setting typically 
exploit \emph{$\ell_2$-subspace embeddings} which asserts that $\mS \mA$ has all directions 
of $\mA$ preserved up to some small relative error.
However, this requires sampling $m=\tilde{O}(d \operatorname{poly} \log d)$ projections:
a stronger condition than retaining only $m < d$ directions as in \acrshort{fd}.
A severe weakness of one-shot randomized sketching in the statistical setting is that only 
one of bias or variance can be well approximated: Classical sketch estimates well the bias but has
significantly higher variance than the optimal solution while Hessian sketch has the 
opposite behaviour \citep{wang2017sketched}.
It is also shown in \cite{wang2017sketched} that averaging the solutions to many sketched 
ridge regression problems can improve the bias and variance estimation.
However, we are primarily interested in the `standard' usage of one-shot sketching such as
Classical and Hessian sketch which use only one sketch of the data.
Thus, a comparison to the so-called \emph{model-averaging} approach falls outside the scope 
of our study.

The central benefit of using \acrshort{fd} for ridge regression is that $o(d^2)$ space is 
required to obtain approximation guarantees \cite{shi2020deterministic}.
At a high level, this is due to the fact \acrshort{fd} incurs less distortion in 
approximating $\covMat{\mA}$ by $\covMat{\mB}$ compared to a random projection 
$\mA^{\top} \mS^{\top} \mS \mA$ when $\mB$ and $\mS \mA$ are of the same size.
Random projections can still preserve much information when $m < d$ for problems such as 
approximate matrix product \citep{cohen2015optimal}; convex constrained least squares 
\citep{pilanci2015randomized, pilanci2016iterative}; underconstrained ridge regression \citep{chowdhury2018iterative}.
However, we are not aware of any \emph{statistical guarantees} on the bias-variance 
tradeoff of ridge regression when $m < d$ dimensions are kept in the sketch.
Results in the optimization setting can be hindered by the need for a new sketch (even if
it is of small-size) at every iteration \cite{pilanci2016iterative}.
Using a single random sketch requires optimizing for the 
correct step size 
and 
is only known to work for Gaussian random projections
\citep{lacotte2019faster} which needs
$O(nd^2)$ time to obtain $\mS \mA$ so is not a viable scalable solution.

\textbf{Our Approach.}
Gaps in the existing literature mark our central departure from current work.
Recall that in the statistical setting the task is to understand the model estimation
(i.e., the mean-square error of $\hat{\vx}$, to be defined formally in Section 
\ref{sec:prelim}) 
meanwhile in the optimization setting we wish to 
minimise the solution estimation error $\|\hat{\vx} - \opt{\vx} \|$.

In both statistical and optimization settings, randomized sketches have proven
the most frequently studied technique.
This is in spite of the superior practical performance \acrshort{fd} provides in
estimating $\covMat{\mA}$ shown in \cite{ghashami2016frequent}.
Where \acrshort{fd} has been studied, there has been no attempt to understand 
the statistical properties (bias, variance, \acrshort{mse}) which is crucial
when using approximate weights $\hat{\vx}$ in place of $\opt{\vx}$.
Secondly, there has been no attempt to understand the performance of 
\acrshort{fd} as a small-space preconditioners for high-accuracy solvers.

\textbf{Paper Outline.}  Section \ref{sec:prelim} outlines the notation
and sketching results we build upon.
In Section \ref{sec:statistical-fdrr} we present the statistical properties of 
(R)\acrshort{fdrr}.
Section \ref{sec:iterative-fdrr} illustrates the iterative ridge sketching method.
Both sections contain experiments to highlight the performance of our methods.
The technical details \& proofs are deferred to the appendix.

\vspace{-3mm}
\section{Preliminaries and Notation}
\label{sec:prelim}
\vspace{-3mm}
Matrices of size $n \times d$ are denoted by uppercase letters e.g.
$\mA \in \R^{n \times d} $.
The $p$-dimensional identity matrix is denoted by $\eye{p}$.
Lower case symbols represent vectors, e.g. $\vx \in \R^{d}$.
The norms we use are the Frobenius norm $\|\mA\|_F$, the spectral or operator
norm over matrices $\|\mA\|_2$ and Euclidean norm over vectors, $\|\vx\|_2$.

\textbf{Frequent Directions: Theoretical Properties}
The property we exploit is that \acrshort{fd} approximately preserves the 
norm of matrix vector products after sketching.
Theorem \ref{thm:fd-guarantee-main} outlines the guarantees obtained by
the returned summary $\mB \in \R^{m \times d}$ of \acrshort{fd} \& \emph{Robust Frequent Directions} (\acrshort{rfd}).
\cite{huang2018near}, show that \acrshort{rfd} improves the accuracy of
\acrshort{fd} by a factor of $2$.
Both implementations are in \Cref{alg:frequent-directions-algorithm}, 
Appendix \ref{app:fd-properties}.

\textbf{Notation for Frequent Directions:} we use
$\Delta_k = \|\mA -\mA_k\|_F^2$ \& $\alpha = \nicefrac{1}{m-k}$ to 
write the bounds for both \acrshort{fd} \& \acrshort{rfd}
\eqref{eq:fd-bound-thm-main}.

\begin{theorem}[\cite{ghashami2016frequent, huang2018near}]
\label{thm:fd-guarantee-main}
Let $\mA \in \R^{n \times d}$. 
The \emph{(Robust) Frequent Directions} algorithm processes
$\mA$ one row at a time, returns a matrix $\mB \in \R^{m \times d}$ and
a scalar $\delta$
such that for any unit vector $\vu \in \R^d$:
\vspace*{-1mm} 
\begin{equation}
    \vspace*{-1mm} 
    \norm{\covMat{\mA} - \left(\covMat{\mB} + \delta \eye{d}\right)}{2}
    \le 
    \alpha' \Delta_k. 
    \label{eq:fd-bound-thm-main}
\end{equation}
If $\fd{\mB}{\mA}$, then $\delta = 0$ \& $\alpha'=\alpha$.
Else if $\rfd{[\mB,\delta]}{\mA}$, $\delta$ is adaptively chosen
and $\alpha'=\alpha/2$.
\end{theorem}

\textbf{Modelling Assumptions.}
We assume that a dataset $\mA \in \R^{n \times d}$ and targets $\vy \in \R^{n}$
are given such that 
\vspace*{-2mm} 
\begin{equation}
    \label{eq:linear_model}
    \vspace*{-2mm} 
    \vy = \mA \vx_0 + \veps.
\end{equation}
For both statistical and optimization settings, we will assume that $n > d$
and the input data has $\rank{\mA} = d$ so that $\opt{\vx}$ is 
uniquely defined.

\noindent
\textbf{Statistical setting.}
The noise 
$\veps$ is zero-mean, $\E(\veps) = \vzero_d$, and the covariance is 
$\E(\veps \veps^{\top}) = \sigma^2 \eye{n}$.
For an estimate of the weights $\hat{\vx}$, we are interested in
\begin{itemize}
    \vspace{-3mm}
    \item{$\bias(\hat{\vx}) = \E(\hat{\vx}) - \vx_0$ \& squared norm
    $\|\bias(\hat{\vx})\|_2^2$.}
    \item{Variance: $\var(\hat{\vx}) = (\hat{\vx} - \E(\hat{\vx}))
    (\hat{\vx} - \E(\hat{\vx}))^{\top}$ and its trace: $\tr(\var(\hat{\vx}))$.}
    \item{Mean-square error (\acrshort{mse}): 
    $\mse(\hat{\vx}) = \E\left(\|\hat{\vx} - \vx_0\|_2^2\right)$ which by the 
    bias-variance decomposition is $\mse(\hat{\vx}) = \|\bias(\hat{\vx})\|_2^2
    + \tr(\var(\hat{\vx}))$.}
\end{itemize}
\vspace{-3mm}
All expectations are taken over the randomness in $\veps$.

\noindent
\textbf{Optimization setting.}
No assumptions on $\veps$ are made and it is assumed to be fixed.
The notion of approximation we adopt is under the Euclidean norm: 
for an estimate $\hat{\vx}$ how small can the \emph{solution error}
$\|\hat{\vx} - \opt{\vx}\|_2$ be made.

\textbf{Randomized Sketching} typically require $\mS \in \R^{m \times n}$ to obtain a  $(1 \pm \rho)$-$\ell_2$ subspace embedding for $\mA$ 
\citep{woodruff2014sketching} which preserves all $\text{rank}(\mA) = d$ directions.
There are many choices of $\mS$ which satisfy the necessary properties 
to compare to the bounds we present \citep{woodruff2014sketching, drineas2016randnla}.
Prior work in both statistical \& optimization perspectives,
does not typically show a \emph{strong} difference in accuracy based 
on how $\mS$ is generated \citep{wang2017sketched, 
pilanci2016iterative, cormode2019iterative}.
Thus, we focus only on 
the \emph{Gaussian} and 
\emph{Sparse Johnson-Lindenstrauss Transforms} (\acrshort{sjlt}) 
\citep{nelson2013osnap}.
The Gaussian is a high-quality sketch and is well-studied due to 
favorable properties such as rotational invariance 
\citep{lacotte2019faster, pilanci2016iterative} yet is slow to apply.
Hence, we also test the \acrshort{sjlt}, which has $s$ nonzeros per 
column so is applied in time $O(s \textsf{nnz}(\mA))$ while also enjoying the same space bound.
Details on constructing the sketches are found in Appendix 
\ref{sec:appendix-miscellaneous}.
\section{Statistical Properties of \acrshort{fdrr}}
\label{sec:statistical-fdrr}
\begin{figure*}[!ht]
    {\centering
    \makebox[0pt]{\includegraphics[width=1.2\textwidth]{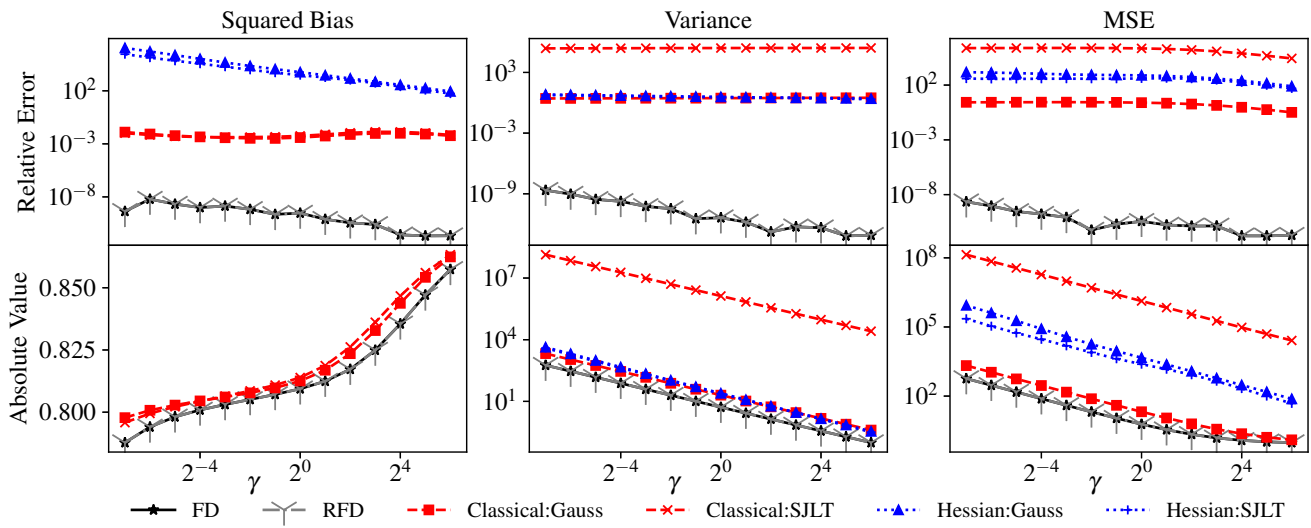}}
    \caption{
    \emph{Relative error} (top) and reported \emph{absolute value} (bottom) of
    the $3$ metrics
    $\|\bias(\hat{\vx})\|^2, 
    \tr(\var(\hat{\vx})), \mse(\hat{\vx})$
    plotted against $\gamma$.
    The instance has effective dimension $R_1 = \lfloor 0.15 d + 0.5 \rfloor$.
    The deterministic methods dominate the randomized methods in all 
    $3$ metrics.
    \emph{Hessian Sketch} has high error in estimating 
    $\|\bias(\opt{\vx})\|_2^2$ so is omitted from the bottom left panel.
    }
    \label{fig:bias-variance-mse}
    }
\end{figure*}

Recall the linear model from Equation \eqref{eq:linear_model} which generates the data \& 
assume $\gamma > 0$ is the regularisation parameter.
Recall that
$\mH_{\gamma} = \covMat{\mA} + \gamma \eye{d}$, the exact solution is
$\opt{\vx} = \mH_{\gamma}^{-1}\mA^{\top}\vy$,
$\hat{\mH} = \covMat{\mB} + \gamma \eye{d}$ and \acrshort{fdrr}
returns $\hat{\vx} = \hat{\mH}^{-1}\mA^{\top}\vy$.
Without sketching we have the following result for the optimal weights:
\begin{lemma}
\label{lem:optimal-bias-variance}
    The optimal bias and variance terms are:
    $\bias(\opt{\vx}) = -\gamma {\mH_{\gamma}^{-1} \vx_0}$ and 
    $\var(\opt{\vx}) = 
    \sigma^2 \mH_{\gamma}^{-1} \covMat{\mA} \mH_{\gamma}^{-1}$
    \label{lem:optimal-bias-variance}
\end{lemma}
\vspace{-4mm}
The proof is given in Appendix 
\ref{sec:appendix-statistical-results}.
Now, the task is to understand the extent to which approximating the weights through
\acrshort{fdrr} distorts the behaviour expressed in Lemma 
\ref{lem:optimal-bias-variance}.
To that end, we have the following lemma which expresses both the bias and 
variance of the weights $\hat{\vx}$ found from solving \acrshort{fdrr}.
\begin{lemma}[\acrshort{fdrr} bias and variance]
\label{lem:fd-bias-variance}
    $\bias(\hat{\vx}) = (\hat{\mH}^{-1} \covMat{\mA} - \eye{d})\vx_0$
    and $\var(\hat{\vx}) = \sigma^2 \hat{\mH}^{-1} \covMat{\mA}
    \hat{\mH}^{-1}$.
\end{lemma}
With this understanding, the next task is to relate these expressions to the 
corresponding terms achieved by $\opt{\vx}$ as expressed in Lemma 
\ref{lem:optimal-bias-variance}.

Observe that we may write $\bias(\hat{\vx})$ as 
$\bias(\hat{\vx}) = (\hat{\mH}^{-1} \covMat{\mA} - \eye{d})\mH_{\gamma} \mH_{\gamma}^{-1}\vx_0$.
This manipulation is useful as $\bias(\opt{\vx}) = -\gamma \mH_{\gamma}^{-1}\vx_0$.
Hence, if we can control the smallest and largest eigenvalues of 
$\mM = (\hat{\mH}^{-1} \covMat{\mA} - \eye{d})\mH_{\gamma}$ then we should be 
able to relate $\|\bias(\hat{\vx})\|_2^2$ to $\|\bias(\opt{\vx})\|_2^2$.
This is exactly how our proof proceeds as we establish the following
\begin{lemma}
\label{lem:spectral-bound}
Let $\gamma' = \gamma - \alpha \Delta_k > 0$.
If 
$\mM = \left(
\hat{\mH}^{-1} \covMat{\mA} -  \eye{d}\right)\mH_{\gamma}$,
then
$\lambda_{\max}(\mM) \le \gamma^2 / \gamma'$ \& 
$\lambda_{\min}(\mM) \ge \gamma'$
\end{lemma}
Given that $\|\mM \vu\|_2 \in [\lambda_{\min}(\mM) \|\vu\|_2,
\lambda_{\max}(\mM) \|\vu\|_2]$ we can take $\vu = \mH_{\gamma}^{-1}\vx_0$ 
combined with Lemma \ref{lem:spectral-bound} to establish:
    $\|\bias(\hat{\vx})\|_2^2 \in  
    \left[ \gamma'^2\|\bias(\opt{\vx})\|_2^2,
    \frac{\gamma^4}{\gamma'^2}\|\bias(\opt{\vx})\|_2^2
    \right].$
Finally, provided that the parameters of the \acrshort{fd} sketch are 
appropriately set compared to the regularization $\gamma$, 
$\|\bias(\hat{\vx})\|_2^2$ can be shown to be within accurate relative-error 
bounds of $\| \bias(\opt{\vx})\|_2^2$.
\begin{theorem}
Let $\fd{\mB}{\mA} \in \R^{m \times d}$ and
let $\theta \in (0,1)$ be a parameter.
If
    $m = \nicefrac{\|\mA - \mA_k\|_F^2}{(1 - \sqrt{1 - \theta})\gamma} + k$,
then
\vspace{-3mm}
\begin{equation*}
    \norm{\bias(\hat{\vx})}{2}^2
    \in 
    \left[
    (1-\theta) \norm{\bias(\opt{\vx})}{2}^2,
    \frac{1}{1-\theta} \norm{\bias(\opt{\vx})}{2}^2
    \right]
\end{equation*}
\label{thm:bias-bound-main}
\end{theorem}
\vspace{-5mm}
Dealing with the variance terms is slightly simpler than the bias terms.
This is thanks to the fact that, 
$\hat{\mH} = \covMat{\mB} + \gamma \eye{d}$ is symmetric positive definite
so we can exploit the \emph{\lowner{}} ordering over such matrices.
Expressing the variance of the weights $\hat{\vx}$ is simple and follows the 
same approach as for the optimal weights $\opt{\vx}$.
Subsequently, we need only invoke standard properties of the \lowner{} ordering
to establish bounds on $\var(\hat{\vx})$ compared to $\var(\opt{\vx})$.
One final nice property of the \lowner{} ordering is that the trace also 
respects the precedence.
That is, if $\mX \preceq \mY$ then $\tr(\mX) \le \tr(\mY)$.
This is the final piece to obtain:
\begin{theorem}
Under the same assumptions as Theorem \ref{thm:bias-bound-main},
$\tr(\var(\hat{\vx})) \in \left[ (1-\theta)\tr(\var(\opt{\vx})),
\frac{1}{1-\theta}\tr(\var(\opt{\vx}))\right]$.
\label{thm:variance-bound-main}
\end{theorem}
Theorems \ref{thm:bias-bound-main} \& \ref{thm:variance-bound-main} immediately
entail the same guarantee on the \acrshort{mse}.
Therefore;
\begin{theorem}
Under the same assumptions as Theorem \ref{thm:bias-bound-main},
    $(1-\theta) \mse(\opt{\vx}) \le 
    \mse(\hat{\vx}) \le 
    \frac{1}{1-\theta}\mse(\opt{\vx}).$
\label{thm:mse-bound-main}
\end{theorem}
\vspace{-2mm}
All proofs for this section are in Appendix 
\ref{sec:bias-variance-fd}, including the extension to obtain a
tighter approximation guarantee with \acrshort{rfd} (\Cref{app:bias-variance-rfd}).
\vspace{-3mm}
\subsection{Experimental Evaluation}
\label{sec:bias-variance-experiments}
\begin{figure*}[!ht]
    {\centering
    \makebox[0pt]{\includegraphics[width=1.2\textwidth]{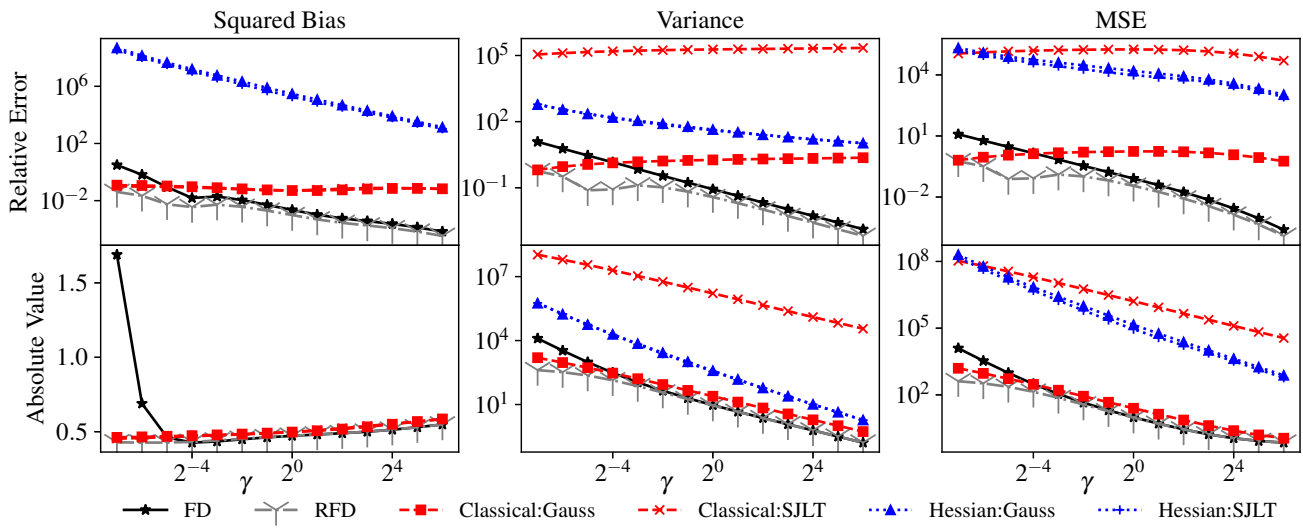}}
    \caption{
    The three metrics
    vs $\gamma$ for $R_2 = \lfloor 0.25 d + 0.5 \rfloor$.
    \acrshort{fdrr}  performs worse than \acrshort{rfdrr}
    at small values of $\gamma$ but then begins to improve.
    Holistically, 
    \acrshort{rfdrr} dominates; \acrshort{fdrr} is next best for large 
    enough $\gamma$;
    the randomized methods each have their deficiencies in bias, variance, or 
    scalability (Gaussian sketch).}
    \label{fig:bias-variance-mse-high-eff-rank}
    }
\end{figure*}
\vspace{-3mm}

\textbf{Competing Methods.}
We compare the \emph{deterministic methods} (Robust) Frequent 
Directions Ridge Regression (R)\acrshort{fdrr} against the 
randomized \emph{Classical} and \emph{Hessian} sketches (Equations 
\eqref{eq:classical-sketch-def}, \eqref{eq:hessian-sketch-def}).
The two methods for generating $\mS$ are \emph{Gaussian} and 
\acrshort{sjlt} with a sparsity of $s=10$.
We refer to the competing methods by \textsf{SketchModel:SketchType},
e.g. Classical:Gaussian.

\textbf{Data Generation.}
We test on synthetic data generated in a similar fashion to 
\cite{shi2020deterministic}.
The data size is $(n,d)=(2^{10},2^9)$ and has \emph{effective rank}
$R = \lfloor r d + 0.5 \rfloor$ for $r \in (0,1)$.
This ensures that most of the energy is concentrated on roughly the top 
$r$-fraction of the directions and is also used to fix the sparsity of the underlying
(and unobserved) ground truth vector $\vx_0$ which generates the data.
We take $\vy = \mA \vx_0 + \veps$ with every $\eps_i \sim \Normal{0}{2^2}$.
Further details for generating the data are in Appendix 
\ref{sec:appendix-miscellaneous}.

\textbf{Experimental Setup.}
We choose $R_1 = \lfloor 0.15 d + 0.5 \rfloor$ and 
$R_2 = \lfloor 0.25 d + 0.5 \rfloor$ so that $R_2$ is of higher effective
rank.
This parameter setting is used to generate the linear model as described
above.
Then we plot the analytical expressions for bias, variance and 
\acrshort{mse} for (R)\acrshort{fdrr} (\Cref{sec:statistical-fdrr}) 
and the randomized
methods.
We set $m = 256$ and vary $\gamma \in \{2^{-8},\dots,2^6\}$ for all
methods.
The random methods are tested 10 times with the median results being reported.
Only one trial is necessary for the deterministic methods.
For the three metrics there is an optimal value $u^*$ which is a function of $\opt{\vx}$
estimated by $\hat{u}$, a function of $\hat{\vx}$.
We measure the \emph{relative error} $|\hat{u} - u^*|/u^*$ and
the \emph{absolute value} of the estimate $\hat{u}$.
Results are reported in Figure \ref{fig:bias-variance-mse}.

\textbf{Findings: $R_1$.}
Across all $3$ metrics, both deterministic methods dominate the randomized methods.
At this projection dimension $m$, \acrshort{rfdrr} is marginally better than 
\acrshort{fdrr}, but the difference in performance negligible.
The relative error of all three metrics is consistently many orders of 
magnitude better than randomized methods.
For the bias, both Classical:Gaussian and Classical:\acrshort{sjlt} method are the most 
competitive;
in absolute terms they are not too far from 
$\|\bias(\opt{\vx})\|^2$ yet
their relative error is much weaker than the deterministic methods.
The Classical:Gaussian sketch appears most consistently competitive to (R)\acrshort{fdrr},
however, this is not scalable for large data streams.
Classical:\acrshort{sjlt} appears competitive for bias but has the worst variance.
On the other hand, both Hessian sketch methods substantially overestimate the 
bias yet their variance is sandwiched between the variance of Classical:Gaussian and Classical:\acrshort{sjlt}.
In the Hessian sketch model, there is little change observed between using Gaussian or \acrshort{sjlt}.

\textbf{Findings: $R_2$.}
The sketch dimension has been maintained at $m=256$.
At this effective dimension we see differences in the deterministic methods as shown in 
Figure \ref{fig:bias-variance-mse-high-eff-rank}.
The relative errors are higher than in Figure \ref{fig:bias-variance-mse} due to the 
increased complexity of the ridge regression problem ($R_2 > R_1$).
\acrshort{rfdrr} remains consistently the best performing sketch across all $3$ metrics.
In relative error,
\acrshort{fdrr} performs up to roughly $2$ orders of magnitude worse than \acrshort{rfdrr} 
and roughly $1$ order of magnitude worse than Classical:Gaussian in bias and variance
up to $\gamma \le 2^{-4}$.
However, for $\gamma > 2^{-4}$ \acrshort{fdrr} begins to perform similarly to 
\acrshort{rfdrr} in both bias and variance.
For the randomized sketches, Classical:Gaussian again looks competitive for small $\gamma$,
yet once roughly $\gamma > 2^{-4}$, there appears to be no improvement in relative error 
and its utility appears to wane, in contrast to (R)\acrshort{fdrr}.
As in Figure \ref{fig:bias-variance-mse}, we observe the same deficiencies with 
Classical:\acrshort{sjlt} and both Hessian sketch methods.

\textbf{Summary.} Across all $3$ metrics and in both the low ($R_1$) and higher ($R_2$)
effective rank regression problems, \acrshort{rfdrr} is the standout sketch method.
For less complex problems ($R_1$), \acrshort{fdrr} is competitive with \acrshort{rfdrr},
however when the complexity of the problem is increased ($R_2$), this behaviour becomes 
dependent on the regularisation.
For the randomized methods, Classical:Gaussian is the most competitive with the 
deterministic methods, but this is fraught with scalability issues as it takes $O(nd^2)$ 
time to generate $\mS \mA$.
When more scalable sketches are used instead of the Gaussian, or the Hessian Sketch 
approach is used, there is noticeable performance degradation.

\vspace{-3mm}
\section{Iterative Frequent Directions Ridge Regression}
\label{sec:iterative-fdrr}
\vspace{-3mm}

\begin{algorithm}[!ht]
\SetAlgoLined
\KwIn{Data $\mA \in \R^{n \times d}$, targets $\vb \in \R^{n}$, 
regularisation $\gamma > 0$, sketch size $m$, 
num. iterations $t \ge 1$,
Method $\textsf{Sk} \in \{\textsf{FD},\textsf{RFD}\}$}
\KwOut{Weights $\hat{\vx} \in \R^{d}$}
$\xfd{\mB,\rho}{\mA}$ 
\Comment{$\rho=0$ if $\textsf{Sk} = \textsf{FD}$ else is nonzero} \\
$\hat{\mH} = \inv{\covMat{\mB}+(\gamma+\rho) \eye{d}},
\vc = \mA^{\top}\vb, \iter{\vx}{0} = \vzero_d$ \\ 
\For{$i=1:t$}{
$\iter{\vx}{i+1} = \iter{\vx}{i} - \hat{\mH}^{-1}
\mA^{\top} \left(\mA \iter{\vx}{t} - \vb\right)$
}
$\hat{\vx} = \iter{\vx}{t}$
 \caption{Iterative Frequent Directions Ridge Regression \acrshort{ifdrr}}
 \label{alg:ifdrr}
\end{algorithm}

\cite{shi2020deterministic}
guarantee a `mid-precision' approximation $\hat{\vx}$ to $\opt{\vx}$.
By that we mean, maintaining 
$m \ge k + \|\mA - \mA_k\|_F^2/\gamma\zeta$ rows in the sketch ensures
error
$\|\hat{\vx} - \opt{\vx}\|_2 \le \zeta \| \opt{\vx}\|_2$.
Thus the sketch grows according to $O(1/\zeta)$ for $\zeta$ accuracy;
this is fine if $\zeta$ is not too small, but
if an application requires the error of $\|\hat{\vx} - \opt{\vx}\|_2$ to
be very small (say $10^{-8}$ or less), then this behaviour is not ideal.

The estimate $\hat{\vx}$ can be refined to better approximate $\opt{\vx}$ 
through iterative gradient steps at the cost of
further passes over the data.
Our proposal (\Cref{alg:ifdrr}) is a Newton-type algorithm that exploits scalable approximation to the Hessian $\mH_{\gamma}$.
Our approach here is reminiscent of many other iterative sketching algorithms 
\citep{pilanci2016iterative,chowdhury2018iterative}.
In common with both of them is that our summary $\mB$ has $o(d)$ rows, a substantial saving over explicitly using the $d \times d$ size
Hessian matrix.
The structure of $\hat{\mH}$ avoids the $O(d^3)$ time cost 
for inversion due to the trick of \cite{shi2020deterministic} or 
Woodbury's Identity.

To prove correctness of Algorithm \ref{alg:ifdrr} we closely follow typical 
proofs for gradient descent-type algorithms.
A key property we need is that the gradient of $f(\vx)$ is 
$\nabla f (\vx) = \mH_{\gamma} ( \vx - \opt{\vx})$ (Lemma \ref{lem:grad-expression},
Appendix \ref{sec:theory-iterative-fdrr}).
Then we are able to analyse the sequence of iterates relative to their distance
from $\opt{\vx}$.
Crucially, we obtain:
\begin{align}
\iter{\vx}{t+1} &= \iter{\vx}{t} - \hat{\mH}^{-1} \nabla f(\iter{\vx}{t}) 
                    \label{eq:ihs_update}\\ 
                &= \iter{\vx}{t} - \hat{\mH}^{-1} \mH_{\gamma}
                (\iter{\vx}{t} - \opt{\vx}).
                \nonumber
\end{align}
Hence, 
$\iter{\vx}{t+1} - \opt{\vx} = 
\left(
\eye{d} - \hat{\mH}^{-1} \mH_{\gamma}
\right)
\left(\iter{\vx}{t} - \opt{\vx}\right).$
Therefore, to show convergence it is enough for us to establish the following 
lemma:
\begin{lemma}
If $ \gamma > 2 {\|\mA - \mA_k\|_F^2}/(m - k),$
then $\norm{\eye{d} - \hat{\mH}^{-1}\mH_{\gamma}}{2} < 1$
\label{lem:spectral-norm-main}
\end{lemma}

\begin{remark}
We claim that the assumption on $\gamma$ in Lemma \ref{lem:spectral-norm-main}
is valid.
Since $m - k \ge 1$ the assumption
asks that $\gamma$ is some fraction of the tail or residual of the mass.
As ridge regression is intended to apply in the high-dimensional setting
with much redundancy in the feature space, it is typical to assume that the 
regularization exceeds the tail in such a fashion.
\end{remark}
The proof of Lemma \ref{lem:spectral-norm-main} is presented in Appendix 
\ref{sec:theory-iterative-fdrr}.
It amounts to manipulating the \acrshort{fd} guarantee of
Theorem \ref{thm:fd-guarantee-main} alongside properties of the \lowner{} ordering.
The starting point is to analyse the spectrum of 
$\eye{d} - \hat{\mH}^{-1} \mH_{\gamma}$.
By matrix similarity we instead analyse 
$\eye{d} - \hat{\mH}^{-1/2} \mH_{\gamma} \hat{\mH}^{-1/2}$ 
but specifically need the 
extremal eigenvalues of the auxiliary matrix 
$\mE = \hat{\mH}^{-1/2} \left( \covMat{\mA} + \gamma \eye{d}\right)\hat{\mH}^{-1/2}.$ 

Crucially, we show that all $\lambda_i(\mE) \in [1,\frac{1}{1-q}]$ 
where $q = \frac{\|\mA - \mA_k\|_F^2}{(m-k)\gamma}$.
This implies that the largest distortion $|1 - \lambda_i(\mE)|$ occurs
at $|1 - \frac{1}{1-q}|$.
Recall that for convergence we required $\|\eye{d} - \hat{\mH}^{-1}\mH\|_2 < 1$
which is satisfied provided $|1 - \frac{1}{1-q}| < 1$.
Hence, we need $q < 1/2$ which is true by the assumption of Lemma 
\ref{lem:spectral-norm-main}.
Finally, we have the convergence theorem which follows by combining all of the 
above pieces.
Details can be found in Appendix \ref{sec:theory-iterative-fdrr}.
\begin{theorem}
\label{thm:iterative-convergence-main}
Let $b \in (0,1/2)$, $\alpha=\nicefrac{1}{m-k}, \Delta_k = \|\mA - \mA_k\|_F^2$
and suppose that $\alpha \Delta_k = b\gamma$.
The iterative sketch algorithm for regression with Frequent Directions satisfies
    $\norm{\iter{\vx}{t+1} - \opt{\vx}}{2} \le 
    \left(\nicefrac{b}{1-b}\right)^{t+1}
    \norm{\opt{\vx}}{2}$
\end{theorem}

Theorem \ref{thm:iterative-convergence-main} demonstrates that 
convergence is governed by an interplay between the regularisation parameter and
the tail of mass.
Let $\beta = \frac{b}{1-b}$ so that 
$\beta = \alpha \Delta_k / (\gamma -  \alpha \Delta_k)$.
When $\beta$ is smaller, decay is faster.
Hence, we can understand the tradeoff between regularisation and
sketch accuracy necessary for convergence.
Decreasing $\beta$ can be achieved by increasing $\gamma$ or by reducing
$\alpha \Delta_k$.
The former regularises the data more (less importance is placed 
on the observed data) while the latter is equivalent to choosing a greater
sketch size.
For example, taking $b=1/4$, Theorem \ref{thm:iterative-convergence-main}
yields $\gamma = 4 \alpha \Delta_k$ so $\beta = 1/3$ \& the error 
decreases by (at worst) a factor of $3$ each iteration.

\begin{remark}
Although $\|\mA - \mA_k\|_F^2$ may not be known (or cannot be estimated)
in advance, setting $k=0$ amounts to taking 
$b = \frac{\|\mA\|_F^2}{m \gamma}$, but this may be too pessimistic in 
practice:
$\|\mA\|_F^2$ can be maintained in small space while observing the stream.
\end{remark}
\begin{figure*}[!ht]
     \centering
     \begin{subfigure}[b]{0.3\textwidth}
         \centering
         \includegraphics[width=\textwidth]{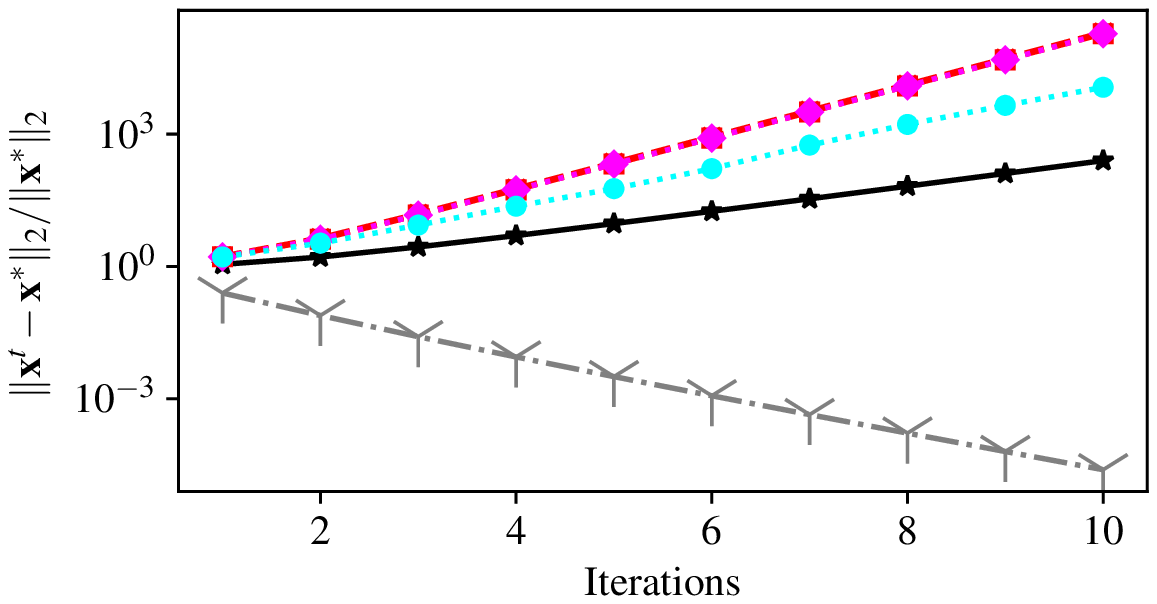}
         \caption{$\gamma=10$}
         \label{fig:w8a-gamma-10}
     \end{subfigure}
     \hfill
     \begin{subfigure}[b]{0.3\textwidth}
         \centering
         \includegraphics[width=\textwidth]{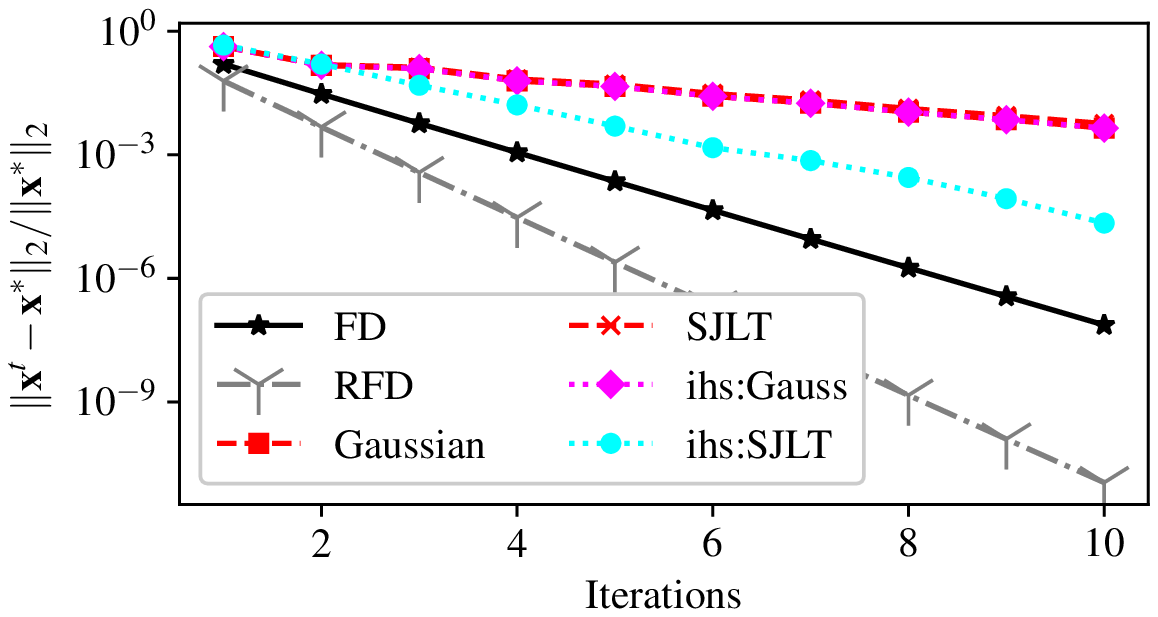}
         \caption{$\gamma=100$}
         \label{fig:w8a-gamma-100}
     \end{subfigure}
     \hfill
     \begin{subfigure}[b]{0.3\textwidth}
         \centering
         \includegraphics[width=\textwidth]{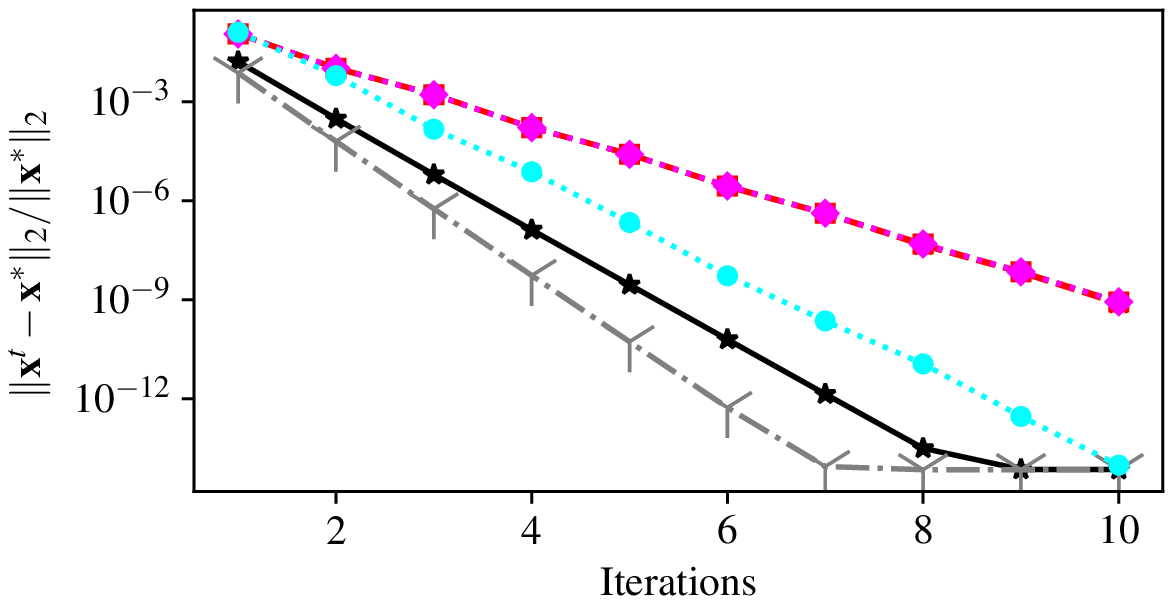}
         \caption{$\gamma=1000$}
         \label{fig:w8a-gamma-1000}
     \end{subfigure}
        \caption{Algorithm \ref{alg:ifdrr} on W8A dataset for $\gamma=10,100,1000$.
        Our approaches, \acrshort{fd} and \acrshort{rfd} outperform the 
        randomized methods.
        The nearest competitor is \acrshort{ihs}:\acrshort{sjlt} which 
        requires a new sketch for every gradient step.
        Our method requires only a single sketch.}
        \label{fig:iterative-sketch-w8a}
\end{figure*}
\vspace{-3mm}
\subsection{Improving Performance with \acrshort{rfd}}
\vspace{-3mm}
One downside of Theorem \ref{thm:iterative-convergence-main} is the fairly
stringent assumption $2 \alpha \Delta_k < \gamma$.
While this is valid, it would be preferable to
weaken this constraint.
Indeed, this is possible due to the improved sketch quality of Robust Frequent Directions.
Theorem \ref{thm:rfd-iterative-convergence-main}
weakens the assumption of $2 \alpha \Delta_k < \gamma$
to ask for $\alpha \Delta_k < \gamma,$ while simultaneously improving the
rate of convergence from $\nicefrac{b}{1-b}$ to $\nicefrac{b}{2-b}$.
Recalling the previous example of taking $b=1/4$, this is an improvement from
$\beta = 1/3$ by Theorem \ref{thm:iterative-convergence-main} to 
$\beta = 1/7$.

\begin{theorem}
\label{thm:rfd-iterative-convergence-main}
Let $b \in (0,1)$ and suppose that $\alpha \Delta_k = b\gamma$.
The iterative sketch algorithm for regression with Robust Frequent Directions satisfies
    $\norm{\iter{\vx}{t+1} - \opt{\vx}}{2} \le 
    \left(\nicefrac{b }{2 - b}\right)^{t+1}
    \norm{\opt{\vx}}{2}.$
\end{theorem}
Due to the theory established for Theorem 
\ref{thm:iterative-convergence-main}, we can essentially repeat the 
proof, adjusting
for the necessary constants which arise due to using
the \acrshort{rfd} sketch $\covMat{\mB} + \delta \eye{d}$ instead of $\covMat{\mB}$.

\vspace{-3mm}
\subsection{Experimental Evaluation.} 
\vspace{-3mm}
\textbf{Setup.} 
All methods were tested over $10$ 
iterations using $m = 256$ rows to generate the sketch.
We generate approximations to $\mH_{\gamma}$ 
using the \emph{deterministic methods} \acrshort{fd} and \acrshort{rfd}.
We also test Algorithm \ref{alg:ifdrr} with \emph{randomized methods}:
the first is to generate a new sketch $\iter{\mS}{t}$ for every iteration
$t$ and set 
$\iter{\tilde{\mH}}{t} = \mA^{\top} {\mS}^{(t)\top}{\iter{\mS}{t} \mA} + \gamma \eye{d}$.
This is exactly the \emph{Iterative Hessian Sketch} (\acrshort{ihs}) technique
of \cite{pilanci2016iterative}.
The second generates a \emph{single} approximation $\tilde{\mH}$
to $\mH_{\gamma}$ and is a modification of \acrshort{ihs} requiring only one
sketch \citep{lacotte2019faster}.
Technically, using a single random sketch requires the tuning of a step size
parameter but for comparison to our method we set the step size to $1$.
We choose $\mS$ to be Gaussian or an \acrshort{sjlt} and refer to the 
randomized approaches as \acrshort{ihs}:Gaussian, 
\acrshort{ihs}:\acrshort{sjlt} for \acrshort{ihs} methods or 
Single:Gaussian \& Single:\acrshort{sjlt} when only a single sketch is used.

\textbf{Datasets.}
We tested on the YearPredictionsMSD, ForestCover \citep{asuncion2007uci} \& 
W8A datasets \citep{chang2011libsvm}.
We take the first $n = 10^5$ samples; these datasets are 
low dimensionality so we expand the feature space using Random 
Fourier Features \citep{rahimi2008random} into $d = 1024$
using the \texttt{RBFSampler} with default settings from \texttt{scikit-learn} \citep{pedregosa2011scikit}.

\textbf{Findings.}
We include the results for the W8A datasets in Figure 
\ref{fig:iterative-sketch-w8a}.
Since the behaviour is consistent across all three datasets, we defer the 
plots for YearPredictions and ForestCover datasets to Appendix \ref{sec:iterative-extra-experiments}.
We found that in line with Theorems \ref{thm:iterative-convergence-main} and 
\ref{thm:rfd-iterative-convergence-main}, convergence was easier for all methods
when $\gamma$ was increased.
When $\gamma = 10$ (the smallest value), all methods except for 
\acrshort{rfd} diverged.
When $\gamma=100$, all methods began to descend towards the optimum.
At any fixed number of iterations the \acrshort{rfd} sketch performed best.
After $10$ iterations \acrshort{rfd} achieved error better than 
$10^{-10}$, followed secondly by \acrshort{fd} which achieved error of 
approximately $10^{-7}$.

Next best was the \acrshort{ihs}:\acrshort{sjlt}, this is 
interesting for two reasons, firstly, the deterministic methods performed better
than all randomized methods, and secondly, because the deterministic methods 
use only a \emph{single} sketch, whereas the best randomized methods uses 
a \emph{new sketch} for every gradient step!
It appears that there is roughly a $1-2$ order of magnitude difference between
\acrshort{fd} and \acrshort{ihs}:\acrshort{sjlt}.
This relative difference is slightly less than the difference between using 
\acrshort{rfd} and \acrshort{fd}.
The methods Single:Gaussian, Single:\acrshort{sjlt} and 
\acrshort{ihs}:Gaussian all perform poorly at this projection dimension of 
$m = 256$.
When we increase $\gamma$ to $1000$, all methods begin to approach the optimum
more rapidly than $\gamma = 10,100$, but again \acrshort{rfd} is the 
stand out winner.
The ordering between the sketch methods established when $\gamma = 10,100$ is
repeated at $\gamma = 1000$ and similarly, this behaviour is reflected on all
the datasets we tested.

In summary, if one requires a high-accuracy solution to the ridge regression
problem, Algorithm \ref{alg:ifdrr} should be employed with a single 
\acrshort{rfd} sketch.
On the examples we tried, this consistently outperformed using \acrshort{fd},
a single sketch, or refreshed random projections.

\section{Conclusion}
We have shown that \acrshort{fd} and \acrshort{rfd} can be analysed from the 
statistical perspective for sketched regression.
Using properties of the sketch we have demonstrated that \acrshort{fdrr} and 
\acrshort{rfdrr} preserves bias, variance and \acrshort{mse} over the weights up
to constant factor relative error.
Similarly, we have shown that both \acrshort{fd} and \acrshort{rfd} can be
employed in the iterated regression model to obtain a highly accurate solution.
In both examples we have shown that \acrshort{fd} performs better than
widely-used random projections.
However, from a practical perspective, \acrshort{rfd} performs the best in 
both the statistical and optimization settings.

\textbf{Acknowledgments.}
CD would like to thank Lee Rhodes, Jon Malkin, Alex Saydakov, Graham Cormode,ang 
for useful discussions and feedback in the preparation of this work.
The work of CD is supported by European Research Council grant ERC-2014-CoG 647557.

\bibliography{fd_bib}

\begin{thebibliography}{26}
\providecommand{\natexlab}[1]{#1}
\providecommand{\url}[1]{\texttt{#1}}
\expandafter\ifx\csname urlstyle\endcsname\relax
  \providecommand{\doi}[1]{doi: #1}\else
  \providecommand{\doi}{doi: \begingroup \urlstyle{rm}\Url}\fi

\bibitem[Agarwal et~al.(2013)Agarwal, Cormode, Huang, Phillips, Wei, and
  Yi]{agarwal2013mergeable}
Agarwal, P.~K., Cormode, G., Huang, Z., Phillips, J.~M., Wei, Z., and Yi, K.
\newblock Mergeable summaries.
\newblock \emph{ACM Transactions on Database Systems (TODS)}, 38\penalty0
  (4):\penalty0 1--28, 2013.

\bibitem[Asuncion \& Newman(2007)Asuncion and Newman]{asuncion2007uci}
Asuncion, A. and Newman, D.
\newblock Uci machine learning repository, 2007.

\bibitem[Avron et~al.(2017)Avron, Clarkson, and Woodruff]{avron2016sharper}
Avron, H., Clarkson, K.~L., and Woodruff, D.~P.
\newblock Sharper bounds for regularized data fitting.
\newblock In \emph{RANDOM}, 2017.

\bibitem[Brand(2002)]{brand2002incremental}
Brand, M.
\newblock Incremental singular value decomposition of uncertain data with
  missing values.
\newblock In \emph{European Conference on Computer Vision}, pp.\  707--720.
  Springer, 2002.

\bibitem[Chang \& Lin(2011)Chang and Lin]{chang2011libsvm}
Chang, C.-C. and Lin, C.-J.
\newblock Libsvm: A library for support vector machines.
\newblock \emph{ACM transactions on intelligent systems and technology (TIST)},
  2\penalty0 (3):\penalty0 1--27, 2011.

\bibitem[Chowdhury et~al.(2018)Chowdhury, Yang, and
  Drineas]{chowdhury2018iterative}
Chowdhury, A., Yang, J., and Drineas, P.
\newblock An iterative, sketching-based framework for ridge regression.
\newblock In \emph{International Conference on Machine Learning}, pp.\
  989--998, 2018.

\bibitem[Clarkson \& Woodruff(2017)Clarkson and Woodruff]{clarkson2017low}
Clarkson, K.~L. and Woodruff, D.~P.
\newblock Low-rank approximation and regression in input sparsity time.
\newblock \emph{Journal of the ACM (JACM)}, 63\penalty0 (6):\penalty0 1--45,
  2017.

\bibitem[Cohen et~al.(2016)Cohen, Nelson, and Woodruff]{cohen2015optimal}
Cohen, M.~B., Nelson, J., and Woodruff, D.~P.
\newblock {Optimal Approximate Matrix Product in Terms of Stable Rank}.
\newblock 55:\penalty0 11:1--11:14, 2016.
\newblock ISSN 1868-8969.
\newblock \doi{10.4230/LIPIcs.ICALP.2016.11}.
\newblock URL \url{http://drops.dagstuhl.de/opus/volltexte/2016/6278}.

\bibitem[Cormode \& Dickens(2019)Cormode and Dickens]{cormode2019iterative}
Cormode, G. and Dickens, C.
\newblock Iterative hessian sketch in input sparsity time.
\newblock In \emph{Neurips Workshop: Beyond First-Order Optimization Methods in
  Machine Learning}, 2019.

\bibitem[De~Klerk(2006)]{de2006aspects}
De~Klerk, E.
\newblock \emph{Aspects of semidefinite programming: interior point algorithms
  and selected applications}, volume~65.
\newblock Springer Science \& Business Media, 2006.

\bibitem[Drineas \& Mahoney(2016)Drineas and Mahoney]{drineas2016randnla}
Drineas, P. and Mahoney, M.~W.
\newblock Randnla: randomized numerical linear algebra.
\newblock \emph{Communications of the ACM}, 59\penalty0 (6):\penalty0 80--90,
  2016.

\bibitem[Ghashami et~al.(2016{\natexlab{a}})Ghashami, Liberty, and
  Phillips]{ghashami2016efficient}
Ghashami, M., Liberty, E., and Phillips, J.~M.
\newblock Efficient frequent directions algorithm for sparse matrices.
\newblock In \emph{Proceedings of the 22nd ACM SIGKDD International Conference
  on Knowledge Discovery and Data Mining}, pp.\  845--854, 2016{\natexlab{a}}.

\bibitem[Ghashami et~al.(2016{\natexlab{b}})Ghashami, Liberty, Phillips, and
  Woodruff]{ghashami2016frequent}
Ghashami, M., Liberty, E., Phillips, J.~M., and Woodruff, D.~P.
\newblock Frequent directions: Simple and deterministic matrix sketching.
\newblock \emph{SIAM Journal on Computing}, 45\penalty0 (5):\penalty0
  1762--1792, 2016{\natexlab{b}}.

\bibitem[Huang(2018)]{huang2018near}
Huang, Z.
\newblock Near optimal frequent directions for sketching dense and sparse
  matrices.
\newblock In \emph{International Conference on Machine Learning}, pp.\
  2048--2057. PMLR, 2018.

\bibitem[Lacotte \& Pilanci(2019)Lacotte and Pilanci]{lacotte2019faster}
Lacotte, J. and Pilanci, M.
\newblock Faster least squares optimization.
\newblock \emph{arXiv preprint arXiv:1911.02675}, 2019.

\bibitem[Liberty(2013)]{liberty2013simple}
Liberty, E.
\newblock Simple and deterministic matrix sketching.
\newblock In \emph{Proceedings of the 19th ACM SIGKDD international conference
  on Knowledge discovery and data mining}, pp.\  581--588, 2013.

\bibitem[Luo et~al.(2019)Luo, Chen, Zhang, Li, and Zhang]{luo2017robust}
Luo, L., Chen, C., Zhang, Z., Li, W.-J., and Zhang, T.
\newblock Robust frequent directions with application in online learning.
\newblock \emph{Journal of Machine Learning Research}, 20\penalty0
  (45):\penalty0 1--41, 2019.
\newblock URL \url{http://jmlr.org/papers/v20/17-773.html}.

\bibitem[Nelson \& Nguy{\^e}n(2013)Nelson and Nguy{\^e}n]{nelson2013osnap}
Nelson, J. and Nguy{\^e}n, H.~L.
\newblock Osnap: Faster numerical linear algebra algorithms via sparser
  subspace embeddings.
\newblock In \emph{2013 ieee 54th annual symposium on foundations of computer
  science}, pp.\  117--126. IEEE, 2013.

\bibitem[Pedregosa et~al.(2011)Pedregosa, Varoquaux, Gramfort, Michel, Thirion,
  Grisel, Blondel, Prettenhofer, Weiss, Dubourg, et~al.]{pedregosa2011scikit}
Pedregosa, F., Varoquaux, G., Gramfort, A., Michel, V., Thirion, B., Grisel,
  O., Blondel, M., Prettenhofer, P., Weiss, R., Dubourg, V., et~al.
\newblock Scikit-learn: Machine learning in python.
\newblock \emph{the Journal of machine Learning research}, 12:\penalty0
  2825--2830, 2011.

\bibitem[Pilanci \& Wainwright(2015)Pilanci and
  Wainwright]{pilanci2015randomized}
Pilanci, M. and Wainwright, M.~J.
\newblock Randomized sketches of convex programs with sharp guarantees.
\newblock \emph{IEEE Transactions on Information Theory}, 61\penalty0
  (9):\penalty0 5096--5115, 2015.

\bibitem[Pilanci \& Wainwright(2016)Pilanci and
  Wainwright]{pilanci2016iterative}
Pilanci, M. and Wainwright, M.~J.
\newblock Iterative hessian sketch: Fast and accurate solution approximation
  for constrained least-squares.
\newblock \emph{The Journal of Machine Learning Research}, 17\penalty0
  (1):\penalty0 1842--1879, 2016.

\bibitem[Rahimi \& Recht(2008)Rahimi and Recht]{rahimi2008random}
Rahimi, A. and Recht, B.
\newblock Random features for large-scale kernel machines.
\newblock In \emph{Advances in neural information processing systems}, pp.\
  1177--1184, 2008.

\bibitem[Shi \& Phillips(2020)Shi and Phillips]{shi2020deterministic}
Shi, B. and Phillips, J.~M.
\newblock A deterministic streaming sketch for ridge regression.
\newblock \emph{arXiv preprint arXiv:2002.02013}, 2020.

\bibitem[van Wieringen(2015)]{van2015lecture}
van Wieringen, W.~N.
\newblock Lecture notes on ridge regression.
\newblock \emph{arXiv preprint arXiv:1509.09169}, 2015.

\bibitem[Wang et~al.(2017)Wang, Gittens, and Mahoney]{wang2017sketched}
Wang, S., Gittens, A., and Mahoney, M.~W.
\newblock Sketched ridge regression: Optimization perspective, statistical
  perspective, and model averaging.
\newblock \emph{The Journal of Machine Learning Research}, 18\penalty0
  (1):\penalty0 8039--8088, 2017.

\bibitem[Woodruff(2014)]{woodruff2014sketching}
Woodruff, D.~P.
\newblock Sketching as a tool for numerical linear algebra.
\newblock \emph{Foundations and Trends® in Theoretical Computer Science},
  10\penalty0 (1–2):\penalty0 1--157, 2014.
\newblock ISSN 1551-305X.
\newblock \doi{10.1561/0400000060}.
\newblock URL \url{http://dx.doi.org/10.1561/0400000060}.

\end{thebibliography}


\appendix  
\onecolumn
\section{Frequent Directions Properties}
\label{app:fd-properties}
\setcounter{algocf}{0}
\renewcommand{\thealgocf}{A\arabic{algocf}}

\begin{algorithm}[!ht]
\SetAlgoLined
\KwIn{Data $\mA \in \R^{n \times d}$, sketch size $m$, 
method $\textsf{Sk} \in \{\textsf{FD},\textsf{RFD}\}$}
\KwOut{$\mB \in \R^{m \times d}$}
Initialise $\mB \leftarrow \vzero_{2m \times d}$ \\ 
$ \rho \leftarrow 0$ \Comment{Parameter for \acrshort{rfd}} \\
\For{$i=1:n$}{
    Insert row $\mA[i,:]$ into all zeros row of $\mB$ \\
    \If{$\mB$ has no zero rows}{
        $\mU, \mSigma, \mV^{\top} = \svd(\mB)$ \\
        $\delta \leftarrow \sigma_m^2$ \\ 
        $\rho \leftarrow \rho + \nicefrac{\delta}{2}$ \\ 
        $\mB \leftarrow \sqrt{\max{(\mSigma^2 - \delta \eye{m},0)}}$
    }
}
\If{$\textsf{Sk} = \textsf{FD}$}{
$\rho \leftarrow 0$ \Comment{$\rho = 0$ for standard \acrshort{fd}} \\
}
\Return{$\mB,\rho$}
 \caption{Frequent Directions (\acrshort{fd}) and Robust Frequent Directions 
 (\acrshort{rfd}) \citep{ghashami2016efficient, huang2018near}.}
 \label{alg:frequent-directions-algorithm}
\end{algorithm}

\begin{algorithm}[!ht]
\SetAlgoLined
\KwIn{Data $\mA \in \R^{n \times d}$, targets $\vb \in \R^{n}$, 
hyperparameter $\gamma > 0$, sketch size $m$, 
method: $\textsf{Sk} \in \{\textsf{FD},\textsf{RFD}\}$}
\KwOut{Weights $\hat{\vx} \in \R^{d}$}
$\vc = \mA^{\top}\vb$ \\
$\xfd{\mB,\rho}{\mA}$ 
\Comment{Call \Cref{alg:frequent-directions-algorithm}: $\rho=0$ iff $\textsf{Sk} = \textsf{FD}$} \\
$\hat{\vx} = \inv{\covMat{\mB} + (\gamma + \rho) \eye{d}} \vc$
 \caption{Frequent Directions Ridge Regression \acrshort{fdrr} \citep{shi2020deterministic}}
 \label{alg:fdrr}
\end{algorithm}

We present the technical details for the results presented in the main body.
Before proceeding to the proofs, we set up some notation and consequences of
the Frequent Directions algorithm.
For $k \ge 0$, let $\mA_k$ denote the optimal rank-$k$ approximation to 
$\mA$.
\begin{theorem}[\cite{ghashami2016frequent}]
\label{thm:fd-guarantee}
Let $\mA \in \R^{n \times d}$. The Frequent Directions algorithm processes
$\mA$ one row at a time and returns a matrix $\mB \in \R^{m \times d}$ 
such that for any unit vector $\vu \in \R^d$:
\[ 0 \le \norm{\mA \vu}{2}^2 -  \norm{\mB \vu}{2}^2 \le 
\frac{\|\mA - \mA_k\|_F^2}{m-k}
\]
\end{theorem}
We will repeatedly use the notation $\Delta_k = \|\mA - \mA_k\|_F^2$.
An equivalent formulation of Theorem \ref{thm:fd-guarantee} is that 
in the \emph{\lowner{}} ordering (see Section 
\ref{sec:linear-algebra-results} for full definitions):
\begin{equation}
    \covMat{\mA} - \frac{\Delta_k}{m - k}\eye{d} \preceq 
    \covMat{\mB} \preceq
    \covMat{\mA}.
\end{equation}
Since each of the 
above matrices is symmetric positive semidefinite we can exploit the 
\lowner{} ordering over such matrices (see Section 
\ref{sec:linear-algebra-results}).
This enables useful properties such as preservation of ordering under the 
following addition of $\gamma \eye{d}$.
Let $\gamma' = \gamma - \nicefrac{\Delta_k}{m - k}$:
\begin{equation}
    \label{eq:hessian-formulation}
    \covMat{\mA} + \gamma'\eye{d} \preceq 
    \covMat{\mB} + \gamma \eye{d} \preceq
    \covMat{\mA} + \gamma \eye{d}
\end{equation}
Which we will denote
\begin{equation}
    \mH_{\gamma'} \preceq 
    \hat{\mH} \preceq
    \mH_{\gamma}.
    \label{eq:hessian-relation}
\end{equation}
We will chiefly manipulate $\hat{\mH} =  \covMat{\mB} + \gamma \eye{d}$ being an approximation to 
$\mH_{\gamma'} = \covMat{\mA} + \gamma \eye{d}$ which is bounded below by
$ \mH_{\gamma'} = \covMat{\mA} + \gamma'\eye{d}$.
The formulation of Equation \eqref{eq:hessian-formulation}
provides the foundation for us to analyse ridge regression
with \acrshort{fd} sketching.
For instance, a basic but key result that underpins our bounds is:
\begin{lemma}
\label{lem:spectral-ordering-for-bounds}
Let $\fd{\mB}{\mA}$ and let $\gamma > 0$ and $\gamma' = 
\nicefrac{\gamma}{m - k} > 0$.
Then
$\gamma' \eye{d} \preceq \covMat{\mB} + \gamma \eye{d} - \covMat{\mA}
\preceq \gamma \eye{d}$
\end{lemma}
\begin{proof}
Let $\alpha = \nicefrac{1}{m - k}$.
Theorem \ref{thm:fd-guarantee} establishes
\begin{equation*}
    \covMat{\mA} - \alpha \Delta_k \eye{d} \preceq \covMat{\mB} \preceq
    \covMat{\mA}.
\end{equation*}
Subtracting $\covMat{\mA}$ from \eqref{eq:hessian-formulation} yields
\begin{equation*}
    - \alpha \Delta_k \eye{d} \preceq \covMat{\mB} - \covMat{\mA} 
    \preceq \vzero_{d \times d}.
\end{equation*}
Adding $\gamma \eye{d}$ establishes the claim.
\end{proof}

\subsection{Relating $\mH_{\gamma}$ to $\mH_{\gamma'}$}
For $\gamma' = \gamma - s > 0$, we will prove Lemma \ref{lem:all-hessian-relations} which relates $\mH_{\gamma'}$ to $\mH_{\gamma}$.
This allows
us to express the lower bound of \eqref{eq:hessian-relation} as 
\begin{equation}
    \frac{\gamma'}{\gamma} \mH_{\gamma} \preceq
    \mH_{\gamma'} \preceq \mH_{\gamma}.
\end{equation}
Note that by properties of the \lowner{} ordering over symmetric positive 
definite matrices, this also implies that the ordering of the eigenvalues is preserved:
\begin{equation}
    \frac{\gamma'}{\gamma} \lambda_i \left( \mH_{\gamma}\right) \preceq
    \lambda_i \left( \mH_{\gamma'}\right) \preceq 
    \lambda_i\left(\mH_{\gamma}\right).
\end{equation}

Before proving the claims which allow us to assert the above, we prove the following 
simple lemma:
\begin{lemma}
\label{lem:simple-bound}
If $x \ge 0$ and let $t > s > 0$, then
\[ \frac{t-s}{t}(x+t) \le x + t - s < x + t.\]
\end{lemma}
\begin{proof}
The upper bound follows trivially since $t - s < t $.
For the lower bound, 
\begin{align*}
    \frac{t - s}{t}(x + t) &= \frac{t - s}{t}x + (t-s) \\
                           &\le x + (t-s)
\end{align*}
since $\frac{t - s}{t} < 1$ and $x \ge 0$.
\end{proof}

\begin{lemma}
\label{lem:all-hessian-relations}
Let $\mA \in \R^{n \times d}$, $\gamma > 0$ and $\mH_{a} = \covMat{\mA} + a \eye{d}$.
If $\gamma' = \gamma - s > 0,$ then
\begin{equation*}
    \frac{\gamma'}{\gamma} \mH_{\gamma} \preceq
    \mH_{\gamma'} \preceq \mH_{\gamma}.
\end{equation*}
\end{lemma}
Before proving Lemma \ref{lem:all-hessian-relations} we focus on the diagonal part of the
\svd{}.
As we operate on a diagonal matrix, we can directly apply 
Lemma \ref{lem:simple-bound}to make the following assertion over 
the singular values of $\mA$.

\begin{lemma}
\label{lem:true-hessian-relationship}
Let $\mSigma^2$ denote the diagonal matrix of singular values of an 
arbitrary input matrix $\mX \in \R^{n \times d}$.
Let $\gamma > 0$ be a regularization parameter from ridge regression 
which ensures that $0 \prec \left(\mSigma^2 + \gamma \eye{d}\right)$.
Suppose that $\gamma' = \gamma - s$ and $\gamma' > 0$.
Then:
\begin{equation}
    \frac{\gamma-s}{\gamma} \left(\mSigma^2 + \gamma \eye{d}\right) \preceq
    \mSigma^2 + \gamma' \eye{d} \preceq
    \mSigma^2 + \gamma \eye{d}.
\end{equation}
\label{lem:diagonal-bound}
\end{lemma}
\begin{proof}
Recall that $\mSigma^2 + a \eye{d} = \textsf{diag}(\sigma_i^2 + a)$
for arbitrary scalar $a \in \R$.
Applying Lemma \ref{lem:simple-bound} on every $\sigma_i^2 + \gamma'$ 
ensures:
\begin{equation}
    \label{eq: inverse-bound}
    \frac{\gamma-s}{\gamma}(\sigma_i^2 + \gamma) 
    \le \sigma_i^2 + \gamma - s
    < \sigma_i^2 + \gamma
\end{equation}
which proves the claim.
\end{proof}

\begin{proof}[Proof of Lemma \ref{lem:all-hessian-relations}]
This is immediate from introducing the orthogonal matrix $\mV$ from the \svd{} of 
$\mA$, Lemma \ref{lem:diagonal-bound}, and the property of the \lowner{} ordering that
$\mC \mX \mC^{\top} \preceq \mC \mY \mC^{\top} $ if and only if 
$\mX \preceq \mY $ ensure that, with $\mC = \mV$:
\begin{equation}
    \frac{\gamma'}{\gamma} \left(\mV \mSigma^2 \mV^{\top} + \gamma \eye{d}\right) 
    \preceq
     \mV \mSigma^2 \mV^{\top} + \gamma' \eye{d} \preceq
     \mV \mSigma^2  \mV^{\top} + \gamma \eye{d}
\end{equation}
that is; 
\begin{equation}
    \frac{\gamma'}{\gamma} \left(\covMat{\mA} + \gamma \eye{d}\right)
    \preceq
     \covMat{\mA} + \gamma' \eye{d} \preceq
     \covMat{\mA} + \gamma \eye{d}
\end{equation}
\end{proof}

Note that Lemma \ref{lem:all-hessian-relations} admits the following overall relations
when $\fd{\mB}{\mA}$ (or $\rfd{[\mC,\delta]}{\mA}$ so that 
$\mB = \covMat{\mC} + \delta \eye{d})$:
\begin{equation}
    \label{eq:all-hessian-order}
    \frac{\gamma'}{\gamma} \left(\covMat{\mA} + \gamma \eye{d}\right)
    \preceq
    \covMat{\mA} + \gamma' \eye{d} \preceq
    \covMat{\mB} + \gamma \eye{d} \preceq
    \covMat{\mA} + \gamma \eye{d}
\end{equation}
\begin{equation}
    \label{eq:inv-hessian-order}
    \inv{\covMat{\mA} + \gamma \eye{d}}
    \preceq
    \inv{\covMat{\mA} + \gamma' \eye{d}} \preceq
    \inv{\covMat{\mB} + \gamma \eye{d}} \preceq
    \frac{\gamma}{\gamma'}
    \inv{\covMat{\mA} + \gamma \eye{d}}
\end{equation}

\section{Statistical Perspectives}
\label{sec:appendix-statistical-results}
We present the technical results from Section \ref{sec:statistical-fdrr}.
Recall from Equation \eqref{eq:linear_model} that we have the following
model
\begin{equation}
    \vy = \mA \vx_0 + \veps
    \label{eq:linear-model-appendix}
\end{equation}
with $\E(\veps) = \vzero_d$ and variance 
$\E(\veps \veps^{\top}) = \sigma^2 \eye{n}$.
A consequence of this linear model is that 
\begin{equation}
    \label{eq:expectation-y}
    \E(\vy) = \mA \vx_0,
\end{equation}
a fact we repeatedly use.

\subsection{Proof of Lemma \ref{lem:optimal-bias-variance}}
Recall that $\opt{\vx} = \mH_{\gamma}^{-1} \mA^{\top} \vy$ 
is the optimal ridge regression solution.
We have the following relations which express the bias, variance, and 
mean-square error (\acrshort{mse}) of $\opt{\vx}$ \emph{without}
sketching.
These have been previously established (see e.g. \cite{van2015lecture})
yet we include them for completeness and consistency of notation.

\begin{lemma}
\label{lem:ridge-bias}
The squared bias of the optimal weights 
$\opt{\vx}$ is:
\begin{equation}
    \norm{\bias(\opt{\vx})}{2}^{2} = 
    \gamma^2 \norm{\mH_{\gamma}^{-1} \vx_0}{2}^2
\end{equation}
\end{lemma}
\begin{proof}
\begin{align}
    \E (\opt{\vx}) &= 
    \E \left( \inv{\covMat{\mA} + \gamma \eye{d}} 
    \mA^{\top} \vy \right) \\ 
    &=
    \inv{\covMat{\mA} + \gamma \eye{d}} 
    \mA^{\top} \mA \vx_0 \\ 
    &=
    \inv{\covMat{\mA} + \gamma \eye{d}} 
    \left(\mA^{\top} \mA + \gamma \eye{d} - \gamma \eye{d}\right) \vx_0 \\ 
    &= 
    \vx_0 - 
    \gamma \inv{\covMat{\mA} + \gamma \eye{d}} \vx_0.
\end{align}
Recalling that
$\bias(\opt{\vx}) = 
\E (\opt{\vx}) - \vx_0$ and taking the 
squared norm recovers the stated result.
\end{proof}
For the variance we have the following:
\begin{lemma}
\label{lem:variance-xopt}
The variance of the optimal weights is:
$\var(\opt{\vx}) = 
\sigma^2 \mH_{\gamma}^{-1} \covMat{\mA} \mH_{\gamma}^{-1}$.
\end{lemma}
\begin{proof}
Recalling from Equation \eqref{eq:expectation-y} that $\E(\vy) = \mA \vx_0$, 
we have 
\begin{align}
    \var(\opt{\vx}) &= 
    \left(\opt{\vx} - \E (\opt{\vx})\right)
    \left(\opt{\vx} - \E (\opt{\vx})\right)^{\top} \\ 
    &=
    \left(\mH_{\gamma}^{-1}\mA^{\top}\vy-
    \mH_{\gamma}^{-1}\mA^{\top}\E(\vy) \right)
    \left(\mH_{\gamma}^{-1}\mA^{\top}\vy-
    \mH_{\gamma}^{-1}\mA^{\top}\E(\vy)\right)^{\top} \\ 
    &= \mH_{\gamma}^{-1} \mA^{\top} 
    \left(\vy - \E(\vy)\right)\left(\vy - \E(\vy)\right)^{\top}
    \mA \mH_{\gamma}^{-1} \\ 
    &= \mH_{\gamma}^{-1} \mA^{\top} \var(\vy)\mA \mH_{\gamma}^{-1}.
\end{align}
Finally, we recognise that $\var(\vy) = \sigma^2 \eye{n}$ which establishes
the claim.
\end{proof}
Using these results for the bias and variance enables the following relationship
for the \emph{mean-squared error}.
Recall that when $\vx_0$ is the vector from the data-generation model, 
\eqref{eq:linear-model-appendix} then the mean-squared error of an estimator 
$\vx$ is defined as 
$\mse(\vx) = \E \| \vx - \vx_0 \|_2^2$.
\begin{lemma}
The mean-squared error of $\opt{\vx}$ is
$\mse(\opt{\vx}) = 
\tr(\var(\opt{\vx})) + \|\bias(\opt{\vx})\|_2^2$.
\end{lemma}
\begin{proof}
We begin from the definition of $\mse(\opt{\vx})$, adding and subtracting
$\E(\opt{\vx})$ in the norm term.
Secondly, recall that from Lemma \ref{lem:ridge-bias} 
$\bias(\opt{\vx}) = \E(\opt{\vx}) - \vx_0$.
Then;
\begin{align*}
\mse(\opt{\vx}) &= 
\E \norm{\opt{\vx} - \E(\opt{\vx}) + \E(\opt{\vx})
- \vx_0}{2}^2 \\ 
&=
\E \norm{\opt{\vx} - \E(\opt{\vx}) + 
\bias(\opt{\vx})}{2}^2 \\
&=
\E \norm{\opt{\vx} - \E(\opt{\vx})}{2}^2 + 
\norm{\bias(\opt{\vx})}{2}^2 \\ 
&= 
\sum_{i=1}^{d} \E ( \sol{\vx}{\gamma (i) } - \E(\opt{\vx})_i)^2 +
\norm{\bias(\opt{\vx})}{2}^2 \\ 
&= 
\sum_{i=1}^{d} \var(\opt{\vx})_{ii} + 
\norm{\bias(\opt{\vx})}{2}^2 \\ 
&= 
\tr(\var(\opt{\vx})) + 
\norm{\bias(\opt{\vx})}{2}^2 \\ 
\end{align*}
\end{proof}
Now that we have the properties on the optimal weights in hand, we can 
relate these to the estimates found from solving the sketched ridge 
problem.

\subsection{Bias-Variance Tradeoff for \acrshort{fd} Sketched Ridge 
Regression: Lemma \ref{lem:fd-bias-variance} - Theorem \ref{thm:mse-bound-main}}
\label{sec:bias-variance-fd}
Recall that for sketched ridge regression the algorithm is roughly:
(i) obtain an \acrshort{fd} sketch $\fd{\mB}{\mA}$; 
(ii) return 
$\hat{\vx} = \inv{\covMat{\mB} + \gamma \eye{d}} \mA^{\top} \vy$.
We will use the shorthand $\hat{\mH} = \covMat{\mB} + \gamma \eye{d}$ 
(which is an approximation to $\mH_{\gamma}$, although we suppress the 
$\gamma$ notation for $\hat{\mH}$).
Our analysis to evaluate the bias and variance roughly follows the same lines as in the preceding section.
However, we need to understand the spectral properties of the sketch
$\mB$.
\begin{lemma}
$\bias(\hat{\vx}) = (\tilde{\mH}^{-1} \covMat{\mA} - \eye{d})\vx_0$
\label{lem:fd-bias-term}
\end{lemma}
\begin{proof}
Observe that 
$\E (\hat{\vx}) = \inv{\covMat{\mB} + \gamma \eye{d}}\covMat{\mA}\vx_0$.
Adding and subtracting $\vx_0$ yields the result.
\end{proof}
The task is now to bound $\|\bias(\hat{\vx})\|_2^2$ in comparison to 
the optimal weights $\|\bias(\opt{\vx})\|_2^2$ found from solving  
unsketched problem.
We need the following lemma which relates the distortion of a matrix-vector
product to the extremal eigenvalues of the matrix.
\begin{lemma}[Extremal distortion of vector norm]
\label{lem:eigenvalues-norm}
Let $\mM \in \R^{d \times d}$ be a symmetric positive definite matrix which has
largest and smallest eigenvalues $\lambda_{\max}(\mM), \lambda_{\min}(\mM)$, 
respectively.
Let $\vu \in \R^{d}$ be arbitrary.
Then
\begin{equation*}
\lambda_{\min}(\mM)\|\vu\|_2\le\| \mM \vu \|_2 \le \lambda_{\max}(\mM)\|\vu\|_2
\end{equation*}
\end{lemma}
\begin{proof}
Follows from eigendecomposition of $\mM$.
\end{proof}
In order to express $\bias(\hat{\vx})$ in terms of 
$\bias(\opt{\vx})$ we multiply by $\mH_{\gamma} \mH_{\gamma}^{-1}$.
That is 
\begin{equation}
\bias(\hat{\vx}) = (\tilde{\mH}^{-1} \covMat{\mA} - \eye{d})\mH_{\gamma}
\cdot \mH_{\gamma}^{-1} \vx_0.
\end{equation}
Now define the matrix $\mM = (\tilde{\mH}^{-1} \covMat{\mA} - 
\eye{d})\mH_{\gamma}$.
Provided that we can control the spectrum of $\mM$, then it will be 
possible to invoke Lemma \ref{lem:eigenvalues-norm}: this is demontstrated
in the subsequent result.
\begin{lemma}
Let $\mM = (\hat{\mH}^{-1} \covMat{\mA} -  \eye{d})\mH_{\gamma}$.
Then
$\lambda_{\max}(\mM) \le \gamma^2 / \gamma'$ and 
$\lambda_{\min}(\mM) \ge \gamma'$
\label{lem:eigenvalue-distortion}
\end{lemma}

\begin{proof}
We will multiply $\mM$ by $-1$ which has the effect of only 
changing the signs but not the magnitude of the eigenvalues.
Then apply the extremal value condition of the generalised Rayleigh quotient (see Section \ref{sec:linear-algebra-results}):
\[\lambda_{\max}(\mM) = \max_{\vu : \|u\|_2=1}|\vu^{\top}
(\eye{d} - \hat{\mH}^{-1} \covMat{\mA})\mH_{\gamma})\vu|\]
from which we can pass the inverse into the denominator (see e.g. Lemma 1
\cite{shi2020deterministic}):
\[
\lambda_{\max}(\mM) = \max_{\vu : \|u\|_2=1}
\left| 
\frac{\vu^{\top} \left(\hat{\mH} - \covMat{\mA}\right)\mH_{\gamma} \vu}
{\vu^{\top} \hat{\mH} \vu}
\right|.
\]
Now apply the variable change $\vz = \mH_{\gamma}^{1/2} \vu$, noting that since
$\mH_{\gamma}$ is symmetric positive definite it has symmetric positive
definite square roots (Section \ref{sec:linear-algebra-results}).
Thus:
\begin{equation}
\lambda_{\max}(\mM) = \max_{\vz}
\left| 
\frac{\vz^{\top} \mH_{\gamma}^{-1/2} \left(\tilde{\mH} - \covMat{\mA}\right)\mH_{\gamma}^{1/2} \vz}
{\vz^{\top}\mH_{\gamma}^{-1/2} \hat{\mH}\mH_{\gamma}^{1/2} \vz}
\right|.
\label{eq:upper-spectral-bound-intermediate}
\end{equation}
To bound the numerator we combat the central term by applying Lemma 
\ref{lem:spectral-ordering-for-bounds} which shows 
$\tilde{\mH} - \covMat{\mA} \preceq \gamma \eye{d}$.
Thus; 
\[
\lambda_{\max}(\mM) \le \gamma \max_{\vz}
\left| 
\frac{\vz^{\top} \mH_{\gamma}^{-1/2} \mH_{\gamma}^{1/2} \vz}
{\vz^{\top}\mH_{\gamma}^{-1/2} \hat{\mH}\mH_{\gamma}^{1/2} \vz}
\right|.
\]
Reverting back to the original coordinates over $\vu$ this is:
\begin{equation}
\lambda_{\max}(\mM) \le \gamma \max_{\vu : \|\vu\|_2=1}
\left| 
\frac{\vu^{\top} \mH_{\gamma} \vu}
{\vu^{\top}\hat{\mH}\vu}
\right|.
\label{eq:upper-spectral-bound-final}
\end{equation}
Now it remains to lower bound the spectrum of the denominator term in 
$\hat{\mH}$.
Again, due to \acrshort{fd} we have $\hat{\mH} \succeq \mH_{\gamma'}$ and Lemma 
\ref{lem:all-hessian-relations}, we know 
$\nicefrac{\gamma'}{\gamma} \mH_{\gamma} \preceq \mH_{\gamma'} \preceq \hat{\mH}$.
Thus, we have 
$\lambda_{\min}(\hat{\mH}) \lambda_{\min}(\mH_{\gamma'}) \ge (\gamma'/\gamma) \lambda_{\min}(\mH_{\gamma})$.
Plugging this into \eqref{eq:upper-spectral-bound-final} yields
$\lambda_{\max}(\mM) \le \gamma^2 / \gamma'$, as required.

For the lower bound we follow essentially the same approach but need to lower 
bound the $\hat{\mH} - \covMat{\mA} $ in
\eqref{eq:upper-spectral-bound-intermediate}, again using 
Lemma \ref{lem:spectral-ordering-for-bounds} to show 
$\lambda_{\min}\left(\hat{\mH} - \covMat{\mA} \right) \ge \gamma'$.
For the denominator, we use Equation \eqref{eq:all-hessian-order} which  reduces the 
absolute value term to $1$.
Finally, the claim follows from Lemma \ref{lem:eigenvalues-norm}.
\end{proof}
We are now in a position to provided constant factor 
approximation bounds for the bias of the estimate 
returned by \acrshort{fd} sketched ridge regression.

\begin{theorem}
\label{thm:bias-bound}
Let $\fd{\mB}{\mA} \in \R^{m \times d}$.
Let $\theta \in (0,1)$ be a parameter and set $c^2 = 1-\theta$.
If
\begin{equation*}
    m = \frac{\|\mA - \mA_k\|_F^2}{(1 - \sqrt{1 - \theta})\gamma} + k,
    \text{~or~}
    \gamma = \frac{\|\mA - \mA_k\|_F^2}{(1 - \sqrt{1 - \theta})(m-k)},
\end{equation*}
then
\begin{equation*}
    \norm{\bias(\hat{\vx})}{2}^2
    \in 
    \left[
    (1-\theta) \norm{\bias(\opt{\vx})}{2}^2,
    \frac{1}{1-\theta} \norm{\bias(\opt{\vx})}{2}^2
    \right]
\end{equation*}
\end{theorem}
\begin{proof}
Denote $\alpha = \nicefrac{1}{m - k}$, $c^2 = 1-\theta$ and 
$\Delta_k = \|\mA - \mA_k\|_F^2$.
The assumptions of the theorem equivalently state that 
$(1-c)\gamma = \alpha \Delta_k$.
We apply Lemma \ref{lem:eigenvalue-distortion} with 
$\vu = \mH_{\gamma}^{-1} \vx_0$ and square all terms so that
\begin{equation*}
    \gamma'^2 \|\vu\|_2^2 \le \|\mM \vu \|_2^2 \le
    \frac{\gamma^4}{\gamma'^2}\|\vu\|_2^2.
\end{equation*}
Since $\gamma' = \gamma - \alpha \Delta_k$ we have $\gamma' = c \gamma$.
Thus;
\begin{equation*}
    c^2 \gamma^2 \|\vu\|_2^2 \le \|\mM \vu \|_2^2 \le
    \frac{\gamma^2}{c^2}\|\vu\|_2^2.
\end{equation*}
Finally, recall that $c^2 = 1 - \theta$ which obtains the stated bound.
\end{proof}

\textbf{Variance.}
The variance term is simpler to analyse thanks to the \lowner{} ordering.
First we illustrate the sketched variance term:
\begin{lemma}
The variance of the weights found from the sketched problem is:
$\var(\hat{\vx}) = 
\sigma^2 \hat{\mH}^{-1} \covMat{\mA} \hat{\mH}^{-1}$.
\end{lemma}

\begin{proof}
We will use from Equation \eqref{eq:linear-model-appendix} that $\E(\vy) = \mA \vx_0$.
\begin{align}
    \var(\hat{\vx}) &= 
    \left(\hat{\vx} - \E (\hat{\vx})\right)
    \left(\hat{\vx} - \E (\hat{\vx})\right)^{\top} \\ 
    &=
    \left(\hat{\mH}^{-1}\mA^{\top}\vy-
    \hat{\mH}^{-1}\mA^{\top}\E(\vy) \right)
    \left(\hat{\mH}^{-1}\mA^{\top}\vy-
    \hat{\mH}^{-1}\mA^{\top}\E(\vy)\right)^{\top} \\ 
    &= \hat{\mH}^{-1} \mA^{\top} 
    \left(\vy - \E(\vy)\right)\left(\vy - \E(\vy)\right)^{\top}
    \mA \hat{\mH}^{-1} \\ 
    &= \hat{\mH}^{-1} \mA^{\top} \var(\vy)\mA \hat{\mH}^{-1}.
\end{align}
Finally, we recognise that $\var(\vy) = \sigma^2 \eye{n}$ which establishes
the claim.
\end{proof}

\begin{theorem}
\label{thm:variance-bound}
Under the same assumptions as Theorem \ref{thm:bias-bound},
\begin{equation*}
    \tr(\var(\opt{\vx})) \le
    \tr(\var(\hat{\vx})) \le 
    \frac{1}{1-\theta}\tr(\var(\opt{\vx}))
\end{equation*}
\end{theorem}
\begin{proof}
Lemma \ref{lem:spectral-ordering-for-bounds} and Equation 
\eqref{eq:inv-hessian-order} 
\begin{equation*}
    \mH_{\gamma}^{-1} \preceq \hat{\mH}^{-1} \preceq
    \mH_{\gamma'}^{-1} \le \frac{\gamma}{\gamma'}\mH_{\gamma}^{-1}.
\end{equation*}
The \lowner{} ordering above ensures that the following is also true:
\begin{equation}
\mH_{\gamma}^{-1} \covMat{\mA} \mH_{\gamma}^{-1}
\preceq 
\hat{\mH}^{-1}\covMat{\mA}\hat{\mH}^{-1}
\preceq
\mH_{\gamma'}^{-1}\covMat{\mA}\mH_{\gamma'}^{-1}.
\end{equation}
Multiplying through by the scalar $\sigma^2$ 
demonstrates that 
\begin{equation*}
    \var(\opt{\vx}) \le 
    \var(\hat{\vx}) \le 
    \var(\sol{\vx}{\gamma'}) \le 
    \frac{\gamma^2}{\gamma'^2}\var(\opt{\vx})
\end{equation*}
Finally, setting $\gamma' = c \gamma$ as in Theorem 
\ref{thm:bias-bound} obtains:
\begin{equation*}
    \var(\opt{\vx}) \le 
    \var(\hat{\vx}) \le 
    \frac{1}{c^2}\var(\opt{\vx}) = 
    \frac{1}{1-\theta}\var(\opt{\vx}).
\end{equation*}
Since the trace maintains the \lowner{} ordering, we have 
established the claim.
\end{proof}

\textbf{\acrshort{mse}.}
We are finally in a position to bound the mean-squared
error from using the \acrshort{fd} ridge estimates 
rather than the exact weights.
The upper bound is immediate since both upper bounds for
bias and variance terms are $\nicefrac{1}{1-\theta}$ multiples of the 
corresponding term over the exact weights.
In Theorem \ref{thm:variance-bound} we have shown a slightly stronger bound
for  $\tr(\var(\opt{\vx}))$ than is necessary here, however, since 
$1 - \theta < 1$, Theorem \ref{thm:variance-bound} implies a lower bound of 
a $(1-\theta) \tr(\var(\opt{\vx}))$ bound.
Hence we obtain
\begin{equation}
    (1-\theta) \mse(\opt{\vx}) \le 
    \mse(\hat{\vx}) \le 
    \frac{1}{1-\theta}\mse(\opt{\vx})
\end{equation}
This is enough to prove the following theorem:
\begin{theorem}
\label{thm:mse-bound}
Under the same assumptions as Theorem \ref{thm:bias-bound},
\begin{equation*}
    (1-\theta) \mse(\opt{\vx}) \le 
    \mse(\hat{\vx}) \le 
    \frac{1}{1-\theta}\mse(\opt{\vx})
\end{equation*}
\end{theorem}

\subsection{Improved Bounds with Robust Frequent Directions}
\label{app:bias-variance-rfd}
The structure of our proof maps allows us to apply it directly to the case when
the Robust Frequent Directions algorithm is employed rather than vanilla 
\acrshort{fd}.
In this case, we have the following analogue of Theorem \ref{thm:fd-guarantee}

\begin{theorem}[\cite{luo2017robust}]
\label{thm:rfd-guarantee}
Let $\mA \in \R^{n \times d}$. 
The Robust Frequent Directions algorithm processes
$\mA$ one row at a time and returns a matrix $\mB \in \R^{m \times d}$ and a 
scalar $\delta \in \R$
such that 
\[
\|\covMat{\mA} - \left(\covMat{\mB} + \delta \eye{d}\right)\|_2
\le \frac{\|\mA - \mA_k\|_F^2}{2(m-k)},
\]
or equivalently
\begin{equation*}
    \covMat{\mA} - \frac{\|\mA - \mA_k\|_F^2}{2(m - k)}\eye{d} \preceq 
    \covMat{\mB} + \delta \eye{d} \preceq
    \covMat{\mA}.
\end{equation*}
\end{theorem}
With this guarantee in place we can easily prove the following theorem:

\begin{theorem}
\label{thm:rfd-bias-bound}
Let $\rfd{\mB}{\mA} \in \R^{m \times d}$.
Let $\theta' \in (0,1)$ be a parameter and set $c'^2 = 1-\theta'$.
If
\begin{equation*}
    m = \frac{\|\mA - \mA_k\|_F^2}{2(1 - \sqrt{1 - \theta'})\gamma} + k,
    \text{~or~}
    \gamma = \frac{\|\mA - \mA_k\|_F^2}{2(1 - \sqrt{1 - \theta'})(m-k)},
\end{equation*}
then
\begin{align*}
    \norm{\bias(\hat{\vx})}{2}^2
    &\in 
    \left[
    (1-\theta') \norm{\bias(\opt{\vx})}{2}^2,
    \frac{1}{1-\theta'} \norm{\bias(\opt{\vx})}{2}^2
    \right] \\ 
    \tr(\var(\hat{\vx})) &\in
    \left[ \tr(\var(\opt{\vx}))
    \frac{1}{1-\theta'}\tr(\var(\opt{\vx}))  
    \right] \\
    \mse(\hat{\vx}) &\in 
    \left[
    (1-\theta') \mse(\opt{\vx}),
    \frac{1}{1-\theta'}\mse(\opt{\vx})
    \right]
\end{align*}
\end{theorem}
\begin{proof}
Let $\hat{\mH}_{\delta, \gamma} = \covMat{\mB} + (\delta + \gamma) \eye{d}$.
Let $\hat{\vx} = \hat{\mH}^{-1}\mA^{\top}\vb$ be the estimated weights.
Again let $\Delta_k = \|\mA - \mA_k\|_F^2$ and $\alpha = \nicefrac{1}{(m - k)}$
We take $\gamma' = \gamma - \nicefrac{\alpha \Delta_k}{2}$.

\noindent
\textbf{Bias Term.}
The same approach as Lemma \ref{lem:fd-bias-term} establishes that 
$\bias(\hat{\vx}) = (\hat{\mH}_{\delta, \gamma}^{-1} \covMat{\mA} - \eye{d})\vx_0$.
Again we use the same trick of multiplying by the identity to obtain
$\bias(\hat{\vx}) = 
(\hat{\mH}_{\delta, \gamma}^{-1} \covMat{\mA} - \eye{d})\mH_{\gamma} \mH_{\gamma}^{-1}\vx_0$.
Thus it suffices to bound the extremal eigenvalues of 
$\mM = (\hat{\mH}_{\delta, \gamma}^{-1} \covMat{\mA} - \eye{d})\mH_{\gamma}$.
We can invoke exactly the same proof as in Lemma \ref{lem:eigenvalue-distortion}
but note that the bounds:
\begin{align}
    \lambda_{\max}(\mM) &\le \frac{\gamma}{\gamma'}\\
    \lambda_{\min}(\mM) &\ge \gamma'. \\
\end{align}
Let $\vu = \mH_{\gamma}^{-1} \vx_0$ and $1-c' = \alpha \Delta_k/2$.
Hence, $\gamma' = c' \gamma$.
Following the proof of Theorem \ref{thm:bias-bound} we establish that:
\begin{equation*}
    c'^2 \gamma^2 \|\vu\|_2^2 \le \|\mM \vu \|_2^2 \le
    \frac{\gamma^2}{c'^2}\|\vu\|_2^2.
\end{equation*}
Recall that $c'^2 = 1 - \theta'$ so that we obtain the stated result.

\noindent
\textbf{Variance and \acrshort{mse}.} Again repeat the argument of Theorem 
\ref{thm:variance-bound} but recall that our altered values of $\gamma/\gamma'$
mean that the upper bound is $\nicefrac{1}{1-\theta'}$.
The \acrshort{mse} result is then immediate, as before by combining the 
bias and variance terms.
\end{proof}

\begin{remark}
To understand the relation between the bias-variance tradeoff using \acrshort{fd}
compared to \acrshort{rfd} we need to account for how $1-\theta$ and $1-\theta'$
interact.
This is observed through some farily simple algebra: 
from the definition of $c$ we can show that $c = \gamma'/\gamma$.
Similarly, 
\[c' = 1 - \frac{\alpha \Delta_k}{2 \gamma}\]
which, by recalling that $\alpha \Delta_k = (1-c)\gamma$ we observe that 
$c' = \frac{1+c}{2}$.
By squaring, we have
\begin{align*}
    1-\theta' &= \left(\frac{1 + c}{2}\right)^2 \\ 
    &= \frac{2 - \theta + 2\sqrt{1-\theta}}{4} \\
    &\ge 1-\theta \text{~for $\theta \in [0,1]$.}
\end{align*}
Therefore, the $1-\theta'$ relative error bounds are tighter than the 
$1-\theta$ bounds.
\end{remark}

\section{Iterative Frequent Directions Ridge Regression: Theory}
\label{sec:theory-iterative-fdrr}
We present the details for the results outlined in Section \ref{sec:iterative-fdrr}.
Before presenting the theory, we set up some notation and some preliminary proofs
to aid the presentation.
Recall Equation \eqref{eq:ridge-regression}
\begin{equation*}
     f(\vx) = \frac{1}{2}\|\mA \vx - \vb \|_2^2
    + \frac{\gamma}{2}\|\vx\|_2^2 
\end{equation*}
and the task is to find, or estimate $\argmin_{\vx} f(\vx)$.
The optimal solution to the above problem is 
\begin{equation}
    \opt{\vx} = \inv{\covMat{\mA} + \gamma \eye{d}}\mA^{\top} \vb.
\end{equation}
The gradient of $f(\vx) $ is 
\begin{equation}
\nabla f(\vx) = \left(\covMat{\mA}+\gamma \eye{d}\right)\vx - \mA^{\top}\vb.    
\end{equation}
Note that $\nabla f(\vx) \in \R^d$ and can be applied in $O(nd)$ time.
That is, $\covMat{\mA}+\gamma \eye{d}$ need not be explicitly computed as the 
matrix-vector products can be evaluated from right to left to avoid the 
matrix-matrix multiplication.
Recall that $\mH_{\gamma} = \covMat{\mA} + \gamma \eye{d}$ is the Hessian matrix of 
second-derivatives of $f(\vx)$.
Computing $\mH_{\gamma}$ requires $O(nd^2)$ time and $O(d^2)$ space.

Rather than computing $\mH_{\gamma}$, we estimate it through the \acrshort{fd} 
sketch.
Recall that $\hat{\mH} = \covMat{\mB} + \gamma \eye{d}$ is our approximation to 
$\mH_{\gamma}$.
Although Algorithm \ref{alg:ifdrr} uses $\mH^{-1}$, this need not be 
computed explicitly and we only need its behaviour as an operator.
This can be understood through the Woodbury inverse lemma which we defer for now 
and present in Section \ref{sec:linear-algebra-results}.
The proof of Theorem \cite{thm:iterative-convergence} roughly follows a standard gradient descent-type proof so we need a few prelimiary results.

\begin{lemma}
$\nabla f (\vx) = \mH ( \vx - \opt{\vx})$
\label{lem:grad-expression}
\end{lemma}

\begin{proof}
\begin{align*}
    \nabla f (\vx) &= \mA^{\top} \left(\mA \vx + \vb\right) + \gamma \vx \\
                  &= \left(\mA^{\top}\mA + \gamma \eye{d}\right) \vx +
                      \mA^{\top}\vb \\ 
                  &= \left(\mA^{\top}\mA + \gamma \eye{d}\right)
                      \left(\vx - \opt{\vx}\right) \\ 
                  &= \mH \left(\vx - \opt{\vx}\right)
\end{align*}
where the penultimate equation follows from the normal equations: 
$\left(\mA^{\top}\mA + \gamma \eye{d}\right)\opt{\vx} = \mA^{\top}\vb$.
\end{proof}

The following lemma represents the current iterate 
$\iter{\vx}{t+1}$ as a function of the previous iterate's distance
from the optimal solution.

\begin{lemma}
The sequence of iterates $\{\iter{\vx}{t+1}\}_{i\ge0}$ follows:
\begin{equation}
\iter{\vx}{t+1} - \opt{\vx} = 
\left(\eye{d} - \hat{\mH}^{-1}\mH\right)
\left(\iter{\vx}{t} - \opt{\vx}\right).
\label{eq:ihs-proof-form}
\end{equation}
\end{lemma}
\begin{proof}
Applying Lemma \ref{lem:grad-expression} to the iterates as 
defined in \eqref{eq:ihs_update} we obtain:
\begin{equation*}
    \iter{\vx}{t+1} - \opt{\vx} = 
    \iter{\vx}{t} - \opt{\vx} - \hat{\mH}^{-1}
    \mH\left(\iter{\vx}{t} - \opt{\vx}\right)
\end{equation*}
which yields the claim after factorisation.
\end{proof}
Taking the norm of both sides of Equation 
\ref{eq:ihs-proof-form} and invoking submultiplicativity we have
\begin{equation*}
\norm{\iter{\vx}{t+1} - \opt{\vx}}{2} \le 
\norm{\eye{d} - \hat{\mH}^{-1}\mH}{2}
\norm{\iter{\vx}{t} - \opt{\vx}}{2}.
\end{equation*}
On the right hand side, the first 2-norm is the spectral norm over
matrices, while the second 2-norm is the Euclidean norm over 
vectors.
Hence, to show $\norm{\iter{\vx}{t+1} - \opt{\vx}}{2} \le 
\norm{\iter{\vx}{t} - \opt{\vx}}{2}$ it suffices to show 
$\norm{\eye{d} - \hat{\mH}^{-1}\mH}{2} < 1$.

\begin{lemma}
If $2 \alpha \Delta_k < \gamma$, then
$\norm{\eye{d} - \hat{\mH}^{-1}\mH}{2} < 1$
\label{lem:spectral-norm}
\end{lemma}

\begin{proof}
Since $\hat{\mH}^{-1}\mH$ is similar to 
$\mE  = \hat{\mH}^{-1/2}\mH \hat{\mH}^{-1/2}$, it has the same 
eigenvalues.
Hence we can bound $\norm{\eye{d} - \mE}{2}$ instead.
By definition, spectral norm asks for:
\begin{equation*}
    \norm{\eye{d} - \mE}{2} = 
    \max_{\vu : \|\vu\|=1}|\vu^{\top}\left(
    \eye{d} - \mE\right)\vu|
\end{equation*}
so we need to show that $\max_{\vu} |1 - \vu^{\top} \mE \vu| < 1$.
To do so, we need a few properties of the \acrshort{fd} sketch.
Let $\alpha = \nicefrac{1}{m-k}$ and $\Delta_k = \|\mA - \mA_k\|_F^2$ so 
that Theorem \ref{thm:fd-guarantee} with the added regularisation ensures
(see Equation \eqref{eq:hessian-formulation}):
\begin{equation}
    \covMat{\mA} + (\gamma - \alpha \Delta_k) \eye{d}
    \preceq
    \covMat{\mB} + \gamma \eye{d}
    \preceq \covMat{\mA} + \gamma \eye{d}.
    \label{eq:useful-lowner}
\end{equation}
Provided that $\gamma > \alpha \Delta_k \eye{d}$, all of the above terms are lower 
bounded by $\vzero_{d \times d}$.
This is equivalent to saying that all eigenvalues are positive,
hence the matrices are full rank and inverses are well-defined.

Denote $\gamma' = \gamma - \alpha \Delta_k$.
Lemma \ref{lem:true-hessian-relationship} shows that 
\begin{equation}
    \frac{\gamma'}{\gamma}\left(\covMat{\mA} + \gamma 
    \eye{d}\right)
    \preceq
    \covMat{\mA} + \gamma' \eye{d}.
\end{equation}
Let $q = \alpha \Delta_k / \gamma > 0$ so that 
$\frac{\gamma'}{\gamma} = 1 - q$.
Invoking  \eqref{eq:useful-lowner} we obtain the ordering:

\begin{equation}
    (1-q) \left(\covMat{\mA} + \gamma \eye{d}\right)
    \preceq
    \covMat{\mB} + \gamma \eye{d}
    \preceq \covMat{\mA} + \gamma \eye{d}.
    \label{eq:lowner-to-bound}
\end{equation}
Now use Point 2 Section \ref{sec:lowner-properties} on all three terms
in \eqref{eq:lowner-to-bound} with $\mC = \hat{\mH}^{-1/2}$.
Again, since all of the matrices in question are symmetric positive definite, 
they have unique symmetric positive definite square roots so we are free to apply
the \lowner{} multiplication order.
\begin{equation}
    (1-q) \hat{\mH}^{-1/2} \left(\covMat{\mA} + \gamma \eye{d}\right) \hat{\mH}^{-1/2}
    \preceq
    \eye{d}
    \preceq \hat{\mH}^{-1/2} \left(\covMat{\mA} + \gamma \eye{d}
    \right)\hat{\mH}^{-1/2}.
\end{equation}
The above equation also implies that $\hat{\mH}^{-1/2} \left(\covMat{\mA} + \gamma \eye{d}\right)
\hat{\mH}^{-1/2} \preceq \frac{1}{1-q} \eye{d}$.
Hence, we also have 
\begin{equation}
    \eye{d}
    \preceq \hat{\mH}^{-1/2} \left(\covMat{\mA} + \gamma \eye{d}
    \right)\hat{\mH}^{-1/2}
    \preceq
    \frac{1}{1-q} \eye{d}.
\end{equation}
The \lowner{} ordering also ensures that
$\lambda_{\min}(\mM) \eye{} \preceq \mM \preceq  \lambda_{\max}(\mM) \eye{}$.
Hence, we have shown that 
\begin{equation}
\lambda_i(\hat{\mH}^{-1/2} \left(\covMat{\mA} + \gamma \eye{d}
    \right)\hat{\mH}^{-1/2} \in \left[1,\frac{1}{1-q}\right].
\end{equation}
Finally, it remains to ensure that 
$\max_{\vu} |1 - \vu^{\top} \mE \vu| < 1$.
Since all $\lambda_i(\mE) \ge 1$, the largest displacement 
occurs at $\lambda_{\max}(\mE)$.
Therefore, $q$ must be set so that
\begin{equation*}
     \left| 1 - \frac{1}{1-q} \right| < 1
\end{equation*}
that is,
\begin{equation}
    \frac{q}{1-q} < 1 \label{eq:q-to-bound}
\end{equation}
which occurs provided $q \in (0,1/2)$ and is thus 
satisfied by the assumption $2 \alpha \Delta_k < \gamma$.
\end{proof}
The preceding result can be used iteratively.
In summary, the following theorem establishes
that choosing $\gamma > 2 \alpha \Delta_k$  ensures the
distance from $\iter{\vx}{t+1}$ to $\opt{\vx}$ is at most 
an $\alpha \Delta_k / \gamma$ factor smaller than that of 
$\iter{\vx}{t}$ to $\opt{\vx}$.

\begin{theorem}
\label{thm:iterative-convergence}
Let $b \in (0,1/2)$ and suppose that $\alpha \Delta_k = b\gamma$.
The iterative sketch algorithm for regression with Frequent Directions satisfies
\begin{equation}
     \norm{\iter{\vx}{t+1} - \opt{\vx}}{2} \le 
    \left(\frac{b}{1-b}\right)^{t+1}
    \norm{\opt{\vx}}{2}
\end{equation}
\end{theorem}
\begin{proof}
Let $\beta = \frac{q}{1-q}$ as in Equation \eqref{eq:q-to-bound}.
Hence, $\beta = \alpha \Delta_k / \gamma'$.
Assuming that 
$\alpha \Delta_k = b\gamma$ so $\gamma' = (1-b)\gamma$ means 
$\beta = \nicefrac{1-c}{c}$.
Since $b < 1/2$ we have $\alpha \Delta_k < \gamma / 2$ hence
Lemma \ref{lem:spectral-norm} establishes that 
$\norm{\eye{d} - \hat{\mH}^{-1}\mH}{2} \le  \beta$.
Thus;
$ \norm{\iter{\vx}{t+1} - \opt{\vx}}{2} \le \beta 
 \norm{\iter{\vx}{t} - \opt{\vx}}{2}$.
By induction, we can iterate this argument to obtain 
$\norm{\iter{\vx}{t+1} - \opt{\vx}}{2} \le 
\beta^{t+1}\norm{\opt{\vx}}{2}$ which follows by recalling 
that $\iter{\vx}{0} = \vzero_d$.
\end{proof}

\subsection{Iterative Ridge Regression with Robust Frequent Directions}
We can slot the robust variant of \acrshort{fd} into the iterative framework.
The proofs follow on as before with a mild adjusting of the constants.
Again, the key technical detail is, for 
$\hat{\mH}_{\delta, \gamma} = \covMat{\mB} + (\delta + \gamma) \eye{d}$,
establishing that $\|\eye{d} - \hat{\mH}_{\delta, \gamma} \mH\|_2 < 1$.
The improvement over using \acrshort{rfd} is that we can weaken the hypothesis
necessary for the result.
\begin{lemma}
If $\alpha \Delta_k < \gamma$, then
$\norm{\eye{d} - \hat{\mH}_{\delta, \gamma}^{-1}\mH}{2} < 1$
\label{lem:rfd-spectral-norm}
\end{lemma}
\begin{proof}
We follow the proof of Lemma \ref{lem:spectral-norm} almost exactly but with the
following modifications.
Equation \eqref{eq:useful-lowner} we use the \acrshort{rfd} guarantee which 
tightens the bounds to 
\begin{equation*}
    \covMat{\mA} + \left(\gamma - \frac{\alpha \Delta_k}{2}\right) \eye{d}
    \preceq
    \covMat{\mB} + \gamma \eye{d}
    \preceq \covMat{\mA} + \gamma \eye{d}.
\end{equation*}
Then take $\gamma' = \gamma - \nicefrac{\alpha \Delta_k}{2}$ and 
$q = \nicefrac{\alpha \Delta_k}{2\gamma}$.
Hence, $\gamma'/\gamma = 1 - q$ as before.
As in Equation \eqref{eq:q-to-bound}, we require $q/(1-q) < 1$ so $q < 1/2$.
By assumption $\alpha \Delta_k < \gamma$ so $q < 1/2$ is satisfied.
\end{proof}

\begin{theorem}
\label{thm:rfd-iterative-convergence}
Let $b \in (0,1)$ and suppose that $\alpha \Delta_k = b\gamma$.
The iterative sketch algorithm for regression with Robust Frequent Directions satisfies
\begin{equation}
    \norm{\iter{\vx}{t+1} - \opt{\vx}}{2} \le 
    \left(\frac{b }{2 - b}\right)^{t+1}
    \norm{\opt{\vx}}{2}
\end{equation}
\end{theorem}
\begin{proof}
Same proof as Theorem \ref{thm:iterative-convergence} except noting that
\begin{align*}
    \beta &= \frac{q}{1-q} \\
    &=\frac{\alpha \Delta_k / 2\gamma}{(2 \gamma - \alpha \Delta_k)/2\gamma} \\
    &= \frac{\alpha \Delta_k}{2\gamma - \alpha \Delta_k} \\ 
    &= \frac{b}{2 - b}
\end{align*}
\end{proof}

\subsection{Further Experimental Results}
\label{sec:iterative-extra-experiments}
In Figures \ref{fig:iterative-sketch-cover} and \ref{fig:iterative-sketch-years}
we plot the results for the experiments as described in Section 
\ref{sec:iterative-fdrr}.
The conclusions remain the same as in Section \ref{sec:iterative-fdrr} with 
Robust Frequent Directions providing the best small-space preconditioner,
followed by Frequent Directions and the iterative Hessian Sketch with 
\acrshort{sjlt} (\acrshort{ihs}:\acrshort{sjlt}).
Although \acrshort{ihs}:\acrshort{sjlt} appears competitive, it requires a new
sketch for every gradient step.
While these can be computed in parallel on viewing the data, it is still many
more sketches than the single sketch required by the (Robust) Frequent 
Directions methods.
When one is restricted to a single sketch, the \acrshort{sjlt} does not perform
at a similar level to the (Robust) Frequent 
Directions methods.

\begin{figure*}
     \centering
     \begin{subfigure}[b]{0.3\textwidth}
         \centering
         \includegraphics[width=\textwidth]{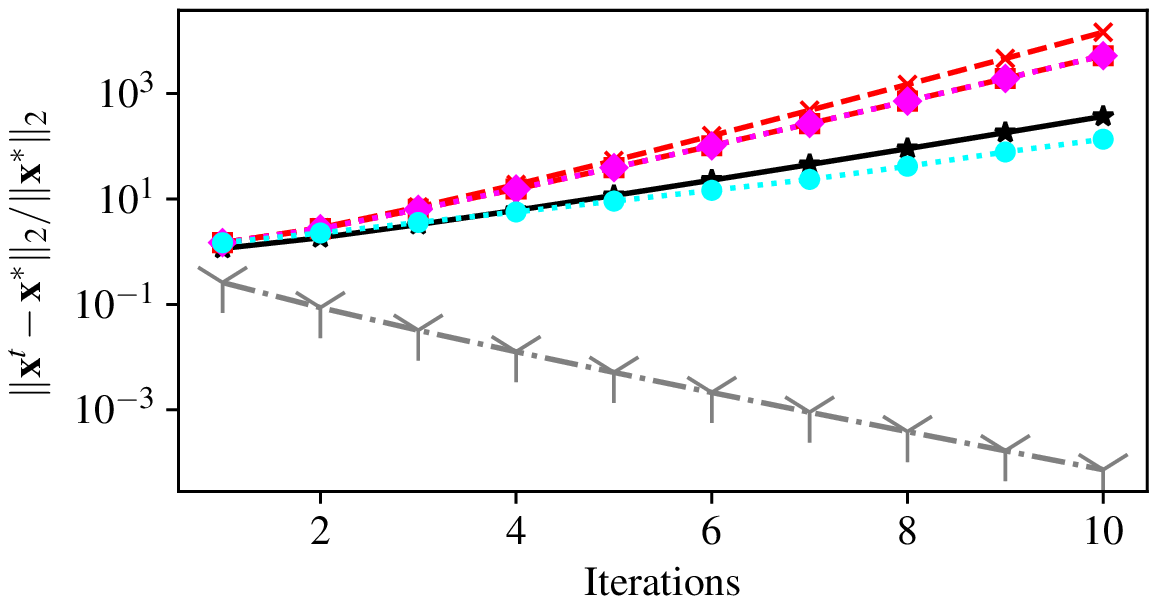}
         \caption{$\gamma=10$}
         \label{fig:cover-gamma-10}
     \end{subfigure}
     \hfill
     \begin{subfigure}[b]{0.3\textwidth}
         \centering
         \includegraphics[width=\textwidth]{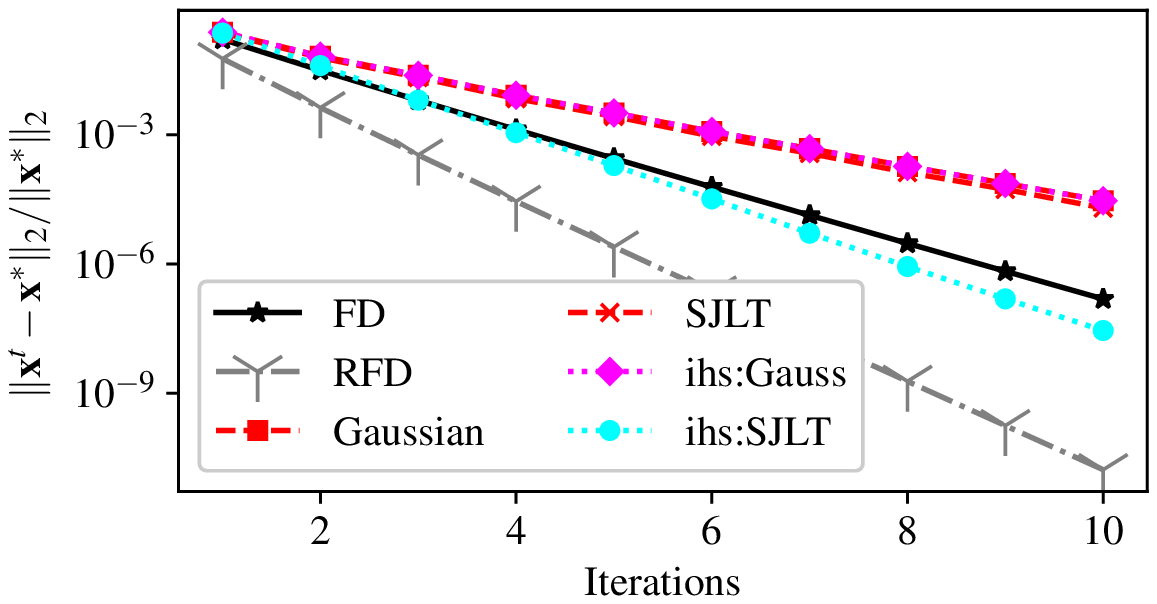}
         \caption{$\gamma=100$}
         \label{fig:cover-gamma-100}
     \end{subfigure}
     \hfill
     \begin{subfigure}[b]{0.3\textwidth}
         \centering
         \includegraphics[width=\textwidth]{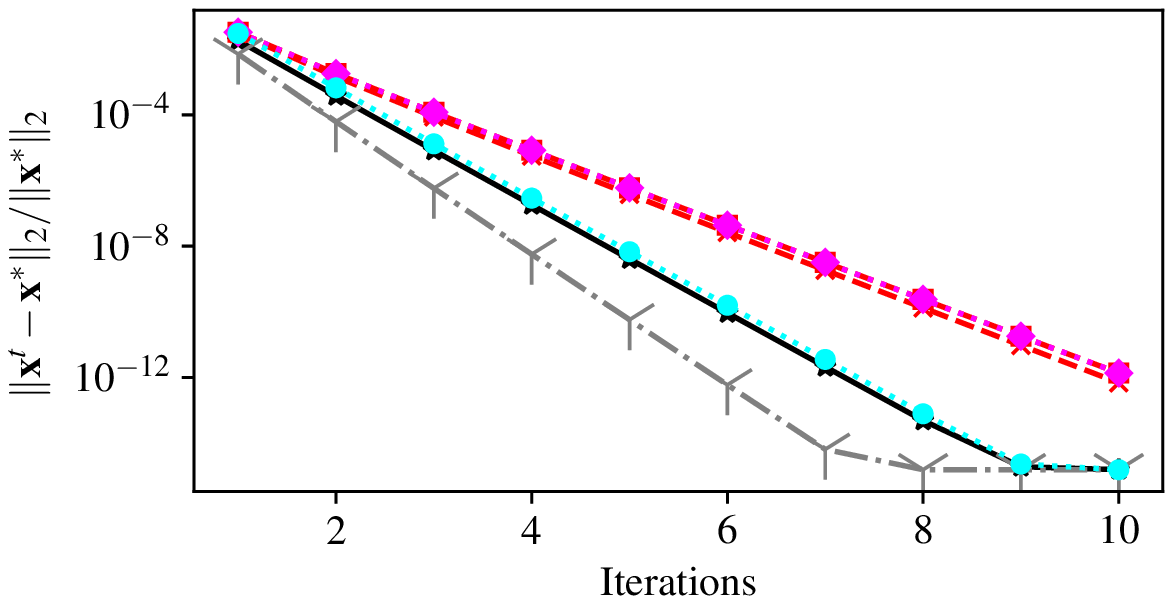}
         \caption{$\gamma=1000$}
         \label{fig:cover-gamma-1000}
     \end{subfigure}
        \caption{Performance of the iterative ridge regression algorithm 
        (Algorithm \ref{alg:ifdrr}) on the ForestCover dataset.}
        \label{fig:iterative-sketch-cover}
\end{figure*}

\begin{figure*}
     \centering
     \begin{subfigure}[b]{0.3\textwidth}
         \centering
         \includegraphics[width=\textwidth]{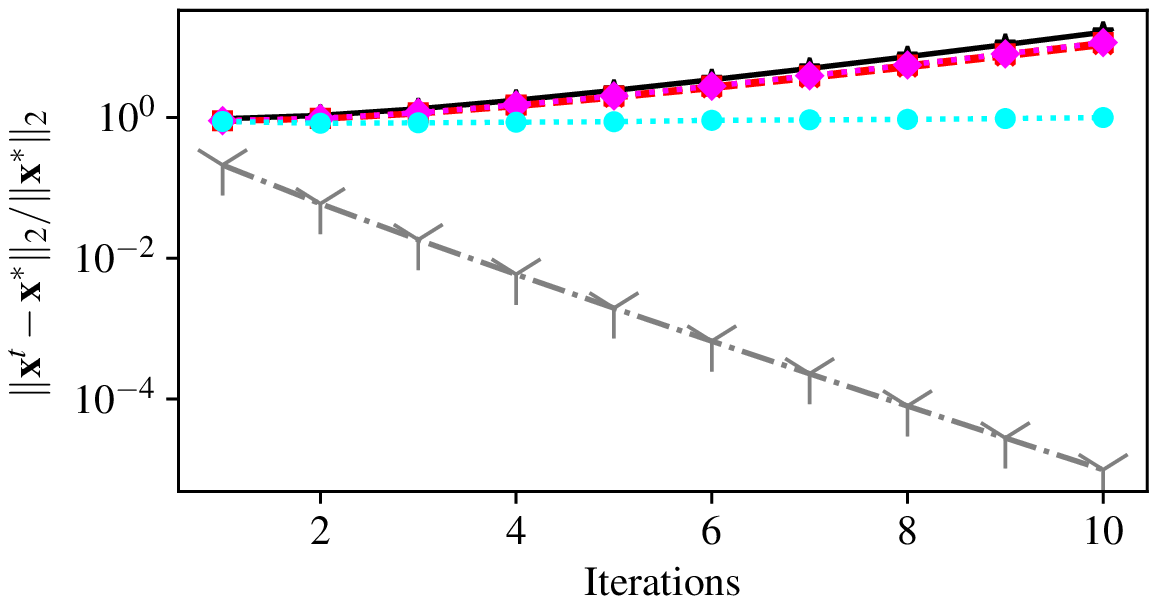}
         \caption{$\gamma=10$}
         \label{fig:years-gamma-10}
     \end{subfigure}
     \hfill
     \begin{subfigure}[b]{0.3\textwidth}
         \centering
         \includegraphics[width=\textwidth]{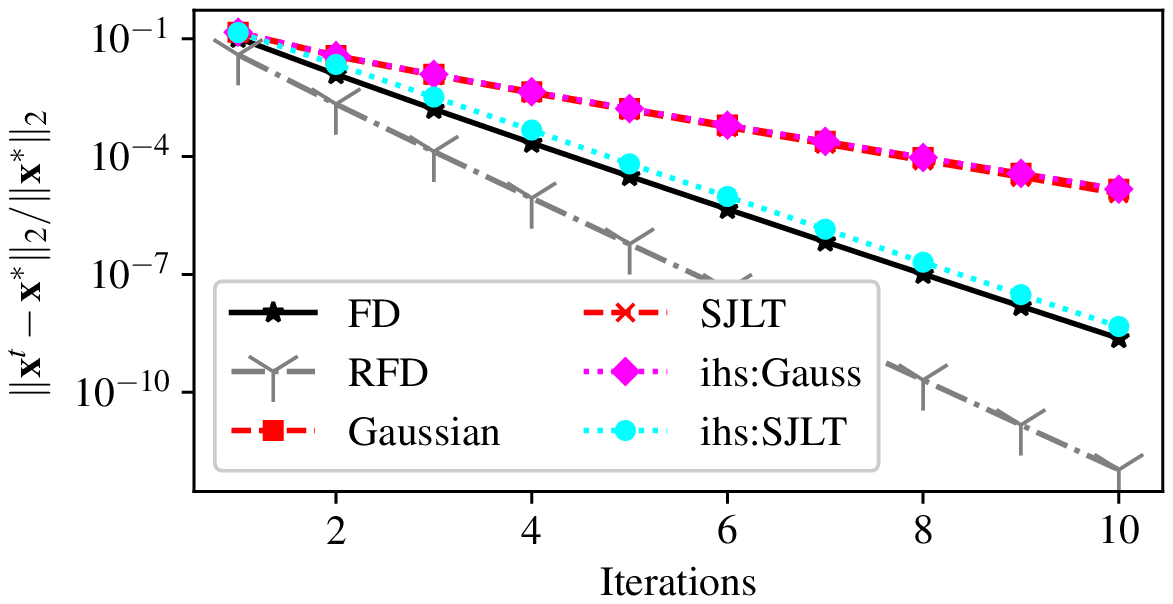}
         \caption{$\gamma=100$}
         \label{fig:years-gamma-100}
     \end{subfigure}
     \hfill
     \begin{subfigure}[b]{0.3\textwidth}
         \centering
         \includegraphics[width=\textwidth]{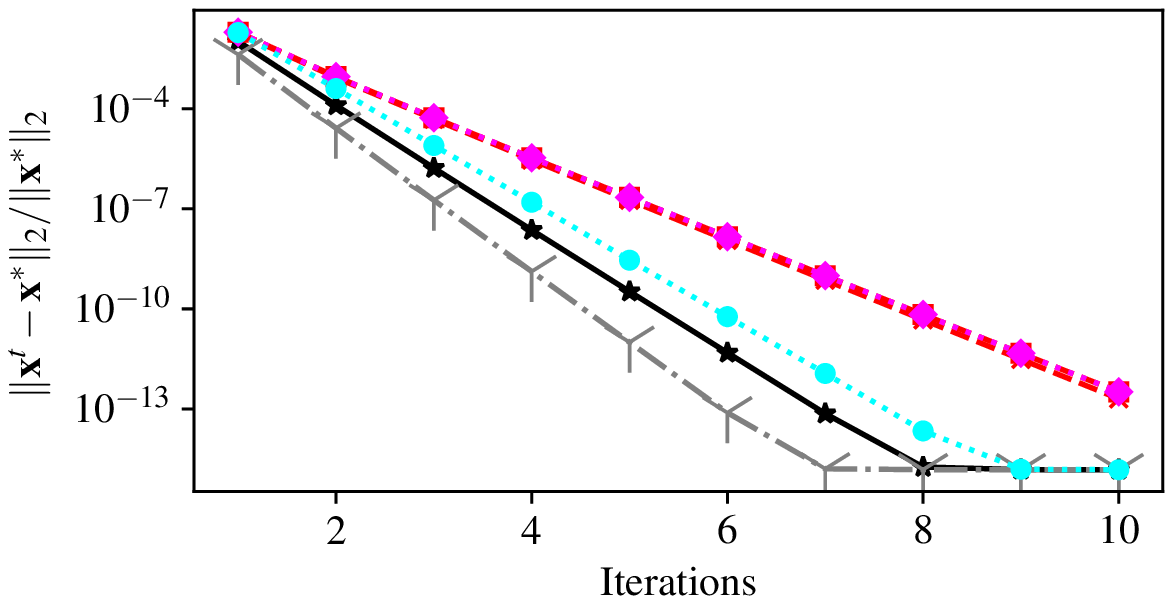}
         \caption{$\gamma=1000$}
         \label{fig:years-gamma-1000}
     \end{subfigure}
        \caption{Performance of the iterative ridge regression algorithm 
        (Algorithm \ref{alg:ifdrr}) on the YearPredictions dataset.}
        \label{fig:iterative-sketch-years}
\end{figure*}

\newpage

\section{Spectral Results and \lowner{} Ordering}
\label{sec:linear-algebra-results}

\subsection{\lowner{} Ordering Properties}
\label{sec:lowner-properties}
A matrix $\mA \in \R^{d \times d}$ is symmetric positive definite (p.d.) if and only 
if it is symmetric and positive definite.
Positive definite means that $\mA  \succ \vzero$, equivalently,
$\vx^{\top} \mA \vx > 0$.
Applied to covariance matrices of full rank, for example 
$\covMat{\mA}$ this is equivalent to asking for $\| \mA \vx \|_2^2 > 0$.
The strictness of each of the above inequalities can be relaxed to allow 
equality if we permit symmetric positive \emph{semi} definite matrices (spsd)
For any two spsd matrices we write $\mB \preceq \mA $ if and only if 
$\mA - \mB \succeq \vzero_{d \times d}$.

Some facts which can be found at (\url{https://www.cs.ubc.ca/~nickhar/W12/NotesMatrices.pdf}) or in 
\emph{Appendix A: Aspects of Semidefinite Programming \citep{de2006aspects}}
are :

\begin{fact}
Let $\mA, \mB, \mC$ be arbtitrary symmetric positive definite matrices.
\begin{enumerate}
    \item{If $\mA \preceq \mB$ then it is not strictly true that 
    $\mA^2 \preceq \mB^2$.  This is the case if the matrices commute, however. }
    \item{If $\mA \preceq \mB$ then $\mC \mA \mC^{\top} \preceq  \mC \mB
    \mC^{\top}$. In fact, this is an if and only if when $\mC$ is of full rank.}
    \item{Let $\lambda_{\min}$ and $\lambda_{\max}$ be the smallest and largest
    eigenvalues of $\mA$. Then $\lambda_{\min} \eye{d} \preceq \mA \preceq
    \lambda_{\max} \mA$.}
    \item{If $\mA \preceq \mB$ then $\tr(\mA) \le \tr(\mB)$}
    \item{If $\mA \preceq \mB$ then $\mB^{-1} \preceq \mA^{-1}$} 
\end{enumerate}
\end{fact}

\subsection{Matrix Results}
We need two further standard results:

\begin{lemma}
A positive definite matrix has a unique positive definite square root
which is symmetric.
\end{lemma}
This lemma allows us to take positive definite matrix $\mQ = \mC \mC^{\top}$
(or its) inverse and 
invoke property 2 for the \lowner{} ordering.
This is because the square root is additionally symmetric so $\mC = \mC^{\top}$.
We repeatedly apply this result on matrices such as $\mH_{\gamma}, 
\mH_{\gamma}^{-1}$.
For square matrices $\mX,\mY$, the generalized Rayleigh quotient is 
$R(\mX, \mY, \vu) = \nicefrac{\vu^{\top} \mX \vu}{\vu^{\top} \mY \vu}$.
\begin{lemma}
The largest (smallest) eigenvalue of a psd matrix $\mX$ maximises (minimises) 
the Rayleight Quotient over $\mX$:

\begin{align*}
    \lambda_{\max} &:= \max_{\vu : \|\vu\|_2=1}R(\mX, \mY, \vu) \\ 
    \lambda_{\min} &:= \min_{\vu : \|\vu\|_2=1}R(\mX, \mY, \vu)
\end{align*}
\end{lemma}

\section{Miscellaneous}
\label{sec:appendix-miscellaneous}
\begin{itemize}
    \item{\textbf{Synthetic Data.} 
    We adapt the synthetic dataset found in Section
    4 \citep{shi2020deterministic}.
    First we set the \emph{effective dimension} $R = \lfloor 0.1\cdot d + 0.5
    \rfloor$.
    This is then used to set the number of nonzero indices in the ground truth
    vector $\vx_0$ and the number of standard deviations for the multivariate
    normal distribution used in generating $\mA$.
    The first $R$ components of $\vx_0$ are sampled from a standard normal 
    distribution, $\vx_0$ is then normalised to unit length.
    The samples (rows) $\mA_i$ are generated by a normal distribution with 
    standard deviation $\exp(-(i-1)^2/R^2)$ for $i = 1:n$.
    Finally, we rotate $\mA$ by a discrete cosine transform.
    We sample noise a noise vector $\veps$ with $\veps_i \sim \Normal{0}{2^2}$
    and set $\vy = \mA \vx_0 + \veps$.
    }
    \item{\textbf{Gaussian Random Projection.}
    The sketch $\mS \mA \in \R^{m \times d}$ is generated by choosing 
    $\mR_{ij} \sim \Normal{0}{1}$ and then taking $\mS = \mR / \sqrt{m}$.
    If $m = O(d \rho^{-2} \log (1/\delta))$, then $\mS$ is a $(1 \pm \rho)$-
    $\ell_2$ subspace embedding for $\mA$ \citep{woodruff2014sketching}.
    }
    \item{\textbf{\acrshort{sjlt}.} 
    We use the \acrshort{sjlt} as it compromises a small sketch dimension $m$
    against the speed at which one can apply the transform.
    The \acrshort{sjlt} is a concatenation of $s$ CountSketch matrices which
    are defined as follows \cite{clarkson2017low}.
    Let $\mS = \vzero_{m \times n}$.  For every column $c \in [n]$, choose a row
    $r$ uniformly at random from $[m]$.
    Randomly set $\mS_{rc} = \pm 1$ each with probability $1/2$.
    It has been shown that such an $\mS$ provides a $(1\pm\rho)$-subspace embedding
    with probability at least $1-\delta$ if $m = O(d^2 \rho^{-2} \delta^{-1})$.
    This is not favourable if $d$ is moderate-to-large and is only suitable for 
    constant probability of success due to the $1/\delta$ dependency.
    However, the CountSketch can be applied to input $\mA$ easily as it is 
    observed.
    In order to retain the benefits of CountSketch but to improve on its weaker
    space dependency, \cite{nelson2013osnap} showed that by 
    stacking $s > O(1/\rho)$ CountSketch matrices of size $m/s$ and choosing 
    $m = O(d \rho^{-2} \operatorname{polylog}(d/\delta))$ then a $(1\pm\rho)$-
    subspace embedding can be achieved.
    The time to apply the embedding is then $s$ times the time required to apply
    a CountSketch, so is still close to the time taken to read the data.
    We refer to this construction as an \acrshort{sjlt} which is a factor of $d$
    better in the projection dimension $m$ and exponentially better for the 
    failure probability than CountSketch.
    Our experiments take $s=10$.
    }
\end{itemize}

\end{document}